\documentclass[12pt]{article}
\usepackage[T1]{fontenc}
\usepackage[latin9]{inputenc}
\usepackage{amssymb,amsthm,tikz,relsize,enumitem,float,stmaryrd}
\usepackage{amsmath}
\allowdisplaybreaks[4]
\usepackage{longtable,array}
\newcolumntype{M}{>{$}l<{$}} 
\usetikzlibrary{shapes.geometric}
\theoremstyle{theorem}

\newtheorem{proposition}{Proposition}
\newtheorem*{proposition*}{Proposition}
\newtheorem{lemma}{Lemma}
\newtheorem{remark}{Remark}
\newtheorem{definition}{Definition}
\title{The logic of KM belief update is contained\\ in the logic of AGM belief revision}
\author{Giacomo Bonanno\\
{\small University of California, Davis, USA}\\
{\small gfbonanno@ucdavis.edu}
}
\date{}

\begin{document}
\maketitle
\begin{abstract}
For each axiom of KM belief update we provide a corresponding axiom in a modal logic containing three modal operators: a unimodal belief operator $B$, a bimodal conditional operator $>$ and the unimodal necessity operator $\square$. We then compare the resulting logic to the similar logic obtained from converting the AGM axioms of belief revision into modal axioms and show that the latter contains the former. Denoting the latter by $\mathcal L_{AGM}$ and the former by $\mathcal L_{KM}$ we show that every axiom of $\mathcal L_{KM}$ is a theorem of $\mathcal L_{AGM}$. Thus AGM belief revision can be seen as a special case of KM belief update. For the strong version of KM belief update we show that the difference between $\mathcal L_{KM}$ and $\mathcal L_{AGM}$ can be narrowed down to a single axiom, which deals exclusively with unsurprising information, that is, with formulas that were not initially disbelieved.     \\[6pt]
\end{abstract}
\section{Introduction}
In \cite{Bon25b} every AGM axiom for belief revision (\cite{AGM85}) was translated into a corresponding axiom in a logic containing three modal operators: a unimodal belief operator $B$, a bimodal conditional operator $>$ and the unimodal necessity (or global) operator $\square$. The interpretation of $B\phi$ is "the agent believes $\phi$",  the interpretation of $\phi>\psi$ is "if $\phi$ is (or were) the case then $\psi$ is (or would be) the case" and the interpretation of $\square\phi$ is "$\phi$ is necessarily true". Letting $K$ be the initial belief set, the fact that $\psi\in K\ast\phi$, that is, that $\psi$ belongs to the revised belief set $K\ast\phi$ prompted by the informational input $\phi$, is expressed in \cite{Bon25b} by the formula $B(\phi>\psi)$ (the agent believes that if $\phi$ were the case then $\psi$ would be the case).

For example, in \cite{Bon25b} AGM axiom $(K\ast 2)$ ($\phi\in K\ast\phi$) is translated into the modal axiom $B(\phi>\phi)$ and  AGM axiom $(K\ast 4)$ (if $\neg\phi\notin K$ then $K+\phi\subseteq K\ast\phi$) is translated into the modal axiom $\big(\neg B\neg\phi \wedge B(\phi\rightarrow\psi)\big)\,\rightarrow\, B(\phi>\psi)$. The translation was done by means of the Kripke-Lewis semantics put forward in \cite{Bon25a}: for every AGM axiom a characterizing property of Kripke-Lewis frames was obtained, which was then shown to characterize a corresponding modal-logic axiom.
\par
In this paper we extend the analysis to the notion of belief update introduced by Katsuno and Mendelzon (KM) in \cite{KatMen91} by providing, for every KM axiom, a corresponding modal axiom. We then show that the modal logic -- denoted by $\mathcal L_{AGM}$ -- obtained by adding the modal AGM axioms to the basic logic, \emph{contains} the logic -- denoted by $\mathcal L_{KM}$ -- obtained by adding the modal KM axioms to the basic logic, in the sense that every axiom of $\mathcal L_{KM}$ is a theorem of logic $\mathcal L_{AGM}$. Thus we confirm the conclusion reached in \cite{Pepetal96} -- for the case of the strong version of KM update -- and reinforced in \cite{Bon25a}, that the notion of AGM belief revision can be viewed as a strengthening of the notion of KM belief update. This conclusion becomes even more conspicuous when comparing the logic of AGM belief revision to the logic of the \emph{strong} version of KM belief update: in this case the difference between the two notions reduces to a single axiom which deals with an item of information $\phi$ that is not surprising, in the sense that it is not the case that, initially, the agent believed $\neg\phi$. 
\par
The paper is structured as follows. In Section \ref{SEC:AGM} we review the KM theory of belief update and characterize each KM axiom in terms of a property of the Kripke-Lewis frames introduced in \cite{Bon25a}. In Section 3 we review the modal logic $\mathcal L$ with the three modal operators $B$, $>$ and $\square$ described above and provide a modal axiom corresponding to each of the frame properties given in Section \ref{SEC:AGM}. In Section \ref{SEC:compare} we recall the modal axiomatization of AGM belief revision introduced in \cite{Bon25b} and compare it to the logic of belief update given in Section 3. Section 5 concludes. All the proofs are given in the Appendix.
\section{The KM theory of belief update}
\label{SEC:AGM}
We consider a propositional logic based on a countable set \texttt{At} of atomic formulas. We denote by $\Phi_0$ the set of Boolean formulas constructed from  \texttt{At} as follows: $\texttt{At}\subset \Phi_0$ and if $\phi,\psi\in\Phi_0$ then $\neg\phi$ and $\phi\vee\psi$ belong to $\Phi_0$. Define  $\phi\rightarrow\psi$, $\phi\wedge\psi$, and $\phi\leftrightarrow\psi$ in terms of $\neg$ and $\vee$ in the usual way (e.g. $\phi\rightarrow\psi$ is a shorthand for $\neg\phi\vee\psi$); furthermore, $\top$ denotes a tautology and $\bot$ a contradiction.
\par
Given a subset $K$ of $\Phi_0 $, its deductive closure $Cn(K)\subseteq\Phi_0$ is defined as follows: $\psi \in Cn(K)$ if and only if there exist $\phi _1,...,\phi _n\in K$ \ (with $n\geq 0$) such that $(\phi _1\wedge ...\wedge \phi _n)\rightarrow \psi $ is a tautology. A set $K\subseteq \Phi_0 $ is \textit{consistent} if $ Cn(K)\neq \Phi_0 $; it is \textit{deductively closed} if $K=Cn(K)$. Given a set $K\subseteq \Phi_0$ and a formula $\phi\in\Phi_0$, the \emph{expansion} of $K$ by $\phi$, denoted by $K+\phi$, is defined as follows: $K+\phi=Cn\left(K\cup\{\phi\}\right)$.
\par
Let $K\subseteq\Phi_0$ be a \emph{consistent and deductively closed} set representing the agent's initial beliefs and let $\Psi \subseteq \Phi_0 $ be a set of formulas representing possible informational inputs. A \emph{belief change function} based on $\Psi$ and $K$ is a function $\circ:\Psi \rightarrow 2^{\Phi_0 }$ ($2^{\Phi_0}$ denotes the set of subsets of $\Phi_0 $) that associates with every formula $\phi \in \Psi $ a set $ K\circ\phi \subseteq \Phi_0 $, interpreted as the change in $K$ prompted by the input $\phi$. We follow the common practice of writing $K\circ\phi$ instead of $\circ(\phi)$ which has the advantage of making it clear that the belief change function refers to a given, \emph{fixed}, $K$. If $\Psi \neq \Phi_0 $ then $\circ$ is called a \emph{partial} belief change function, while if $\Psi =\Phi_0 $ then $\circ$ is called a \emph{full-domain} belief change function. 
 
\subsection{The KM theory of belief update}
We consider the notion of belief update introduced by Katsuno and Mendelzon (KM) in \cite{KatMen91} and, in Section \ref{SEC:compare}, we compare it to the notion of belief revision introduced by Alchourr\'{o}n, G\"{a}rdenfors and Makinson (AGM) in \cite{AGM85}. The formalism in the two theories is somewhat different. In \cite{KatMen91} a belief state is represented by a sentence in a finite propositional calculus and belief update is modeled as a function over formulas, while in \cite{AGM85} a belief state is represented (as we did above) as a set of formulas. \par
Note that, while \cite{KatMen91} allows for the possibility of inconsistent initial beliefs, following \cite{Pepetal96} we take as starting point a consistent belief set. 
\par 
We follow closely the axiomatization of belief update proposed by \cite{Ara25,Pep93,Pepetal96}, which has the advantage of making  update directly comparable to revision (note, however, that \cite{Ara25,Pep93,Pepetal96} only cover the case of "strong" update, where axioms $(K\diamond6)$ and $(K\diamond7)$ are replaced by $(K\diamond9)$: see Section \ref{SEC:compare} below).
\par

Consider the following version of the KM axioms of belief update based on the  \emph{consistent} and deductively closed set $K$ (representing the initial beliefs): $\forall \phi,\psi\in\Phi_0$,
\begin{enumerate}
\item[]$(K\diamond 0)$\quad $K\diamond \phi=Cn(K\diamond \phi)$.
\item[]$(K\diamond 1)$\quad $\phi\in K\diamond \phi$.
\item[]$(K\diamond 2)$\quad If $\phi\in K$ then $K\diamond \phi=K$.
\item[]$(K\diamond 3a)$\quad If $\phi$ is a contradiction then $K\diamond \phi=\Phi_0$.
\item[]$(K\diamond 3b)$\quad If $\phi$ is not a contradiction then $K\diamond \phi\ne\Phi_0$.
\item[]$(K\diamond 4)$\quad If $\phi\leftrightarrow\psi$ is a tautology then $K\diamond \phi=K\diamond \psi$.
\item[]$(K\diamond 5)$\quad $K\diamond (\phi\wedge\psi)\subseteq(K\diamond \phi)+\psi$.
\item[]$(K\diamond 6$)\quad If $\psi\in K\diamond\phi$ and $\phi\in K\diamond\psi$ then $K\diamond\phi=K\diamond\psi$ .
\item[]$(K\diamond 7$)\quad If $K$ is complete
\footnote{\label{FT:MCS }A belief set $K$ is \emph{complete} if, for every formula $\phi\in\Phi_0$, either $\phi\in K$ or $\lnot\phi\in K$.}
\,then $(K\diamond\phi)\cap(K\diamond\psi)\subseteq K\diamond(\phi\vee\psi)$.
\end{enumerate}

\begin{remark}
\label{REM:aboutKM8}
Katsuno and Mendelzon provide an additional axiom (they name it U8), which they call the "disjunction rule". \cite{Ara25,Pep93,Pepetal96} translate it into the following "axiom", which makes use of maximally consistent sets of formulas (MCS), also called possible worlds.\footnote{\label{FT:MCS }
A set of formulas $\Delta\subset \Phi_0$ is maximally consistent if it is consistent and, furthermore, $\forall \phi\in\Phi_0\setminus\Delta$, $\Delta\cup\{\phi\}$ is inconsistent. Every MCS is deductively closed and complete.\\
 If $K$ is consistent, then $K$ is complete if and only if $\llbracket K\rrbracket$ is a singleton, where $\llbracket K\rrbracket$ denotes the set of maximally consistent sets of formulas that contain all the formulas in $K$. 
}\,
 Given a set of formulas $\Gamma\subseteq\Phi_0$, let $\llbracket\Gamma\rrbracket$ be the set of MCS that contain all the formulas in $\Gamma$.\footnote{
 Following the common notation in the literature, we denote by $W$ the set of MCS, or possible worlds, and by $w$ an element of $W$ (hence $w\in W$ is a set of formulas). Thus, $w\in\llbracket\Gamma\rrbracket$ if and only if $w\in W$ and $\Gamma\subseteq w$.}
 \,  The additional axiom is the following (note that $\llbracket K\rrbracket\ne\varnothing$ if and only if $K$ is consistent):
\begin{equation*}
\label{KM8}
  (K\diamond 8)\qquad\text{If } \llbracket K\rrbracket\ne\varnothing  \text{ then } K\diamond\phi=\bigcap\limits_{w\in \llbracket K\rrbracket} (w\diamond\phi).
\end{equation*}
\end{remark}
\noindent
($K\diamond 8$) is of a different nature than the other axioms, since it applies the update operator not only to the initial belief set $K$ but also to the individual MCS contained in $ \llbracket K\rrbracket$. ($K\diamond 8$) can be viewed more as a condition on the interpretation or semantics rather than a real axiom; it is superfluous in our framework since its role is directly captured by the semantics described below; for a more detailed discussion see \cite{Bon25a}.
\par\noindent
$(K\diamond0)$ does not appear in the list of axioms provided by Katsuno and Mendelzon, since their formalism is not in terms of belief sets. For $i\in\{1,2,4,5,6,7\}$, axiom $(K\diamond i)$ is a translation of Katsuno and Mendelzon's axiom $(Ui)$ (for details see \cite{Pep93}). The conjunction of $(K\diamond 3a)$ and $(K\diamond 3b)$ is the translation of Katsuno and Mendelzon's axiom $(U3)$ when attention is restricted to the case where the initial belief set $K$ is consistent.\footnote{ 
Katsuno and Mendelzon allow for the possibility that the initial beliefs are inconsistent, in which case the conjunction of ($K\diamond 3a$) and ($K\diamond 3b$) would be stated as follows: $K\diamond \phi=\Phi_0$ if and only if either $K$ is inconsistent or $\phi$ is a contradiction. It should be noted that one important difference between update and revision is precisely that updating an inconsistent $K$ by a consistent formula $\phi$ yields the inconsistent belief set $\Phi_0$, while revising an inconsistent $K$ by a consistent formula $\phi$ yields a consistent set (AGM axiom ($K\ast5b$): see Section \ref{SEC:compare}).
} 

\par
The following lemma, proved in the Appendix, shows that axiom $(K\diamond 7)$ can be replaced by the following, seemingly stronger, axiom (obtained from  $(K\diamond 7)$ by dropping the clause `if $K$ is complete'):
\begin{equation*}
(K\diamond 7s)\qquad (K\circ\phi)\cap(K\circ\psi)\subseteq K\circ(\phi\vee\psi).
\end{equation*}
\begin{lemma}
\label{LEM:K7s}
$(K\diamond 7s)$ follows from $(K\diamond 7)$ and $(K\diamond 8)$
\end{lemma}
In order to facilitate the conversion to a modal formula, we replace $(K\diamond 6)$ with the following weaker form, which -- in the presence of $(K\diamond0)$ -- is equivalent to $(K\diamond 6)$ (the role of the added clause `$\top\in K\diamond(\phi\wedge\psi)$' is explained in Remark \ref{REM:topinKdiamond}):
\[(K\diamond 6w)\quad \text{If } \psi\in K\diamond\phi \text{ and } \phi\in K\diamond\psi \text{ and }\top\in K\diamond(\phi\wedge\psi)\text{ then }K\diamond\phi=K\diamond\psi\] 
\par
In view of Lemma \ref{LEM:K7s} and the above observation that $(K\diamond 8)$ is superfluous in our framework, we define a KM belief update as follows.
\begin{definition}
\label{DEF:update}
A \emph{KM belief update function},  based on the  consistent and deductively closed set $K$, is a full domain belief change function \, $\diamond:\Phi_0\to 2^{\Phi_0}$ that satisfies axioms $(K\diamond 0)$-$(K\diamond 5)$, $(K\diamond 6w)$ and $(K\diamond 7s)$.
\end{definition}
Katsuno and Mendelzon  show that, semantically, their notion of belief update corresponds to partial pre-orders on the set of maximally consistent sets of formulas. We will, instead, make use of the Kripke-Lewis semantics introduced in \cite{Bon25a}, to which we now turn.
\subsubsection{Kripke-Lewis semantics}
\begin{definition}
\label{DEF:frame}
A \emph{Kripke-Lewis frame} is a triple $\left\langle {S,\mathcal B,f} \right\rangle$ where
\begin{enumerate}
\item $S$ is a set of \emph{states}; subsets of $S$ are called  \emph{events}.
\item $\mathcal B \subseteq S \times S$ is a binary relation on $S$ which is serial: $\forall s\in S, \exists s'\in S$, such that $s\mathcal B s'$ ($s\mathcal B s'$ is an alternative notation for $(s,s')\in\mathcal B$). We denote by $\mathcal B(s)$ the set of states that are reachable from $s$ by $\mathcal B$: $\mathcal B(s)=\{s'\in S: s\mathcal B s'\}$. $\mathcal B(s)$ is interpreted as the set of states that, initially, the agent considers doxastically possible at state $s$.
\item $f:S\times (2^S\setminus\varnothing) \rightarrow 2^S$ is a \emph{selection function} that associates with every state-event pair $(s,E)$ (with $E\ne\varnothing$) a set of states $f(s,E)\subseteq S$,  interpreted as the set of states that are closest (or most similar) to $s$, conditional on event $E$.
    \end{enumerate}
\end{definition}
\noindent
We require seriality of the belief relation because it ensures that the initial beliefs at state $s$, represented by the non-empty set $\mathcal{B}(s)$, are consistent. Similarly, the requirement that $f(s,E)$ is defined only if $E\ne\varnothing$ ensures that the informational input is consistent (how to deal with inconsistent information is discussed in Section 3). Note the absence, in Definition \ref{DEF:frame}, of the extra properties imposed on the selection function in \cite{Bon25a}, namely, 
\begin{enumerate}
\item[] $ f(s,E)\subseteq E$ (Identity).
    \item[] $ f(s,E)\ne\varnothing$ (Normality).
    \item[] if $s\in E$ then $s\in f(s,E)$ (Weak Centering).
    \end{enumerate}
The reason why we do not impose any properties on the selection function is that we want to highlight the role of each property in the characterization of the KM axioms. For example, Identity (localised to $\mathcal{B}(s)$) plays a role in the characterization of $(K\diamond2)$ but not in the characterization of the other axioms, Normality (localised to $\mathcal{B}(s)$) is used to characterize axiom $(K\diamond 3b)$ but not the other axioms.\footnote{In conditional logic, Identity corresponds to the axiom "if $\phi$ then $\phi$", denoted by $\phi >\phi$, and Normality validates the axiom $(\phi >\psi)\rightarrow \neg(\phi >\neg\psi)$ for consistent $\phi$: see \cite{Lew71}.}
\par
Adding a valuation to a frame yields a model. Thus a \emph{model} is a tuple $\left\langle {S,\mathcal B,f,V} \right\rangle$ where $\left\langle {S,\mathcal B,f} \right\rangle$ is a frame and $V:\texttt{At}\rightarrow 2^S$ is a valuation that assigns to every atomic formula $p\in\texttt{At}$ the set of states where $p$ is true.
\begin{definition}
\label{Truth0}
Given a model $M=\left\langle {S,\mathcal B,f,V} \right\rangle$ define truth of a Boolean formula $\phi\in\Phi_0$ at a state $s\in S$ in model $M$, denoted by $s\models_M\phi$, as follows:
\begin{enumerate}
\item if $p\in\texttt{At}$ then $s\models_M p$ if and only if $s\in V(p)$,
\item $s\models_M\neg\phi$ if and only if $s\not\models_M\phi$,
\item $s\models_M(\phi\vee\psi)$ if and only if $s\models_M\phi$ or $s\models_M\psi$ (or both).
\end{enumerate}
\end{definition}
\noindent
We denote by $\Vert\phi\Vert_M$ the truth set of formula $\phi$ in model $M$: $\Vert\phi\Vert_M=\{s\in S: s\models_M\phi\}.$
\par
Given a model $M=\left\langle {S,\mathcal B,f,V} \right\rangle$ and a state $s\in S$,  let $K_{s,M}=\{\phi \in \Phi _{0}:\mathcal B(s) \subseteq \Vert \phi \Vert_{M} \}$; thus a formula $\phi$ belongs to $K_{s,M}$ if and only if at state $s$ the agent believes $\phi$, in the sense that $\phi$ is true at every state that the agent considers doxastically possible at $s$. We identify $K_{s,M}$ with the agent's \emph{initial beliefs at state} $s$. It is shown in \cite{Bon25a} that the set $K_{s,M}\subseteq\Phi_0$ so defined is deductively closed and consistent (the latter property follows from seriality of the belief relation $\mathcal B$).
\par
 Next, given a model $M=\left\langle {S,\mathcal B,f,V} \right\rangle$ and a state $s\in S$, let $\Psi_M=\{\phi\in\Phi_0:\Vert\phi\Vert_M\ne\varnothing\}$\footnote{
Since, in any given model, there are formulas $\phi$ such that $\Vert\phi\Vert_M=\varnothing$ (at the very least all the contradictions), $\Psi_M$ is a proper subset of $\Phi_0$.
  }\ 
and define the following \emph{partial} belief change function
 $\circ:\Psi_M\to 2^{\Phi_0}$ based on $\Psi_M$ and $K_{s,M}$:
\begin{equation}
\tag{RI}
\label{RI}
\begin{array}{*{20}{ll}}
 \psi\in K_{s,M}\circ \phi&\text{if and only if, }\,\, \forall s'\in\mathcal B(s), \, f\left(s',\Vert\phi\Vert_M\right)\subseteq\Vert\psi\Vert_M\\[10pt]
 {}&\text{or, equivalently, }\bigcup\limits_{s'\in\mathcal B(s)}{f\left(s',\Vert\phi\Vert_M\right)}\,\subseteq\,\Vert\psi\Vert_M
\end{array}
\end{equation}
Given the customary interpretation of selection functions in terms of conditionals, (RI) can be interpreted as stating that $\psi\in K_{s,M}\circ\phi$ if and only if at state $s$ the agent believes that "if $\phi$ is (were) the case then $\psi$ is (would be) the case".\footnote{Note that we allow for both the indicative and the subjunctive conditional. The indicative form (if $\phi$ \emph{is} the case then $\psi$ \emph{is} the case) seems to be more appropriate when the initial belief set does not contain $\neg\phi$ (that is, if the agent initially considers $\phi$ possible), while the subjunctive form (if $\phi$ \emph{were} the case then $\psi$ \emph{would be} the case) seems  to be more appropriate when the agent initially believes $\neg\phi$.}\, This interpretation will be made explicit in the modal logic considered in Section 3.
\begin{remark}
\label{REM:topinKdiamond}
Consider an arbitrary model $M$ and state $s$ and let $\circ:\Psi_M\to 2^{\Phi_0}$ be the partial belief change function defined by \eqref{RI}. Then, for every consistent formula $\chi\in\Phi_0$,
\[\Vert\chi\Vert_M\ne\varnothing\text{ if and only if } \top\in K_{s,M}\diamond\chi.\footnote{
Proof. If $\top\in K_s\diamond\chi$  then $\chi\in\Psi_M$ and thus, by definition of $\Psi_M$, $\Vert\chi\Vert_M\ne\varnothing$. Conversely, if $\Vert\chi\Vert_M\ne\varnothing$ then  $\chi\in\Psi_M$ and, for every $s'\in\mathcal B(s)$, $f(s',\Vert\chi\Vert_M)$ is well-defined and, trivially, $\bigcup\limits_{s'\in\mathcal B(s)}{f(s',\Vert\chi\Vert_M)}\subseteq S=\Vert\top\Vert$, so that $\top\in K_s\diamond\chi$.
}\] 
\end{remark}
In what follows, when stating an axiom for a belief change function, we implicitly assume that it applies to every formula \emph{in its domain}. For example, the axiom $\phi\in K\circ\phi$ asserts that, for all $\phi$ in the domain of $\circ$, \,$\phi\in K\circ\phi$.
\begin{definition}
\label{DEF:validity}
An axiom for belief change functions is \emph{valid on a frame $F$} if, for every model based on that frame and for every state $s$ in that model, the partial belief change function defined by \eqref{RI} satisfies the axiom. An axiom is \emph{valid on a set of frames $\mathcal F$} if it is valid on every frame $F\in\mathcal F$.
\end{definition}
A stronger notion than validity is that of frame correspondence. The following definition mimics the notion of frame correspondence in modal logic.
\begin{definition}
\label{DEF:correspondence}
We say that an axiom $A$ of belief change functions \emph{is characterized by}, or \emph{corresponds to}, or \emph{characterizes}, a property $P$ of frames if the following is true:
\begin{enumerate}[label=(\arabic*)]
\item axiom $A$ is valid on the class of frames that satisfy property $P$, and
\item if a frame does not satisfy property $P$ then axiom $A$ is not valid on that frame, that is, there is a model based on that frame and a state in that model where the partial belief change function defined by \eqref{RI} violates axiom $A$.
\end{enumerate}
\end{definition}
\pagebreak
\begin{figure}[H]
\begin{center}
\small{
$\renewcommand{\arraystretch}{2}\arraycolsep=3pt\begin{array}{*{20}{l}}
\qquad\text{\large KM axiom }&{}&\qquad\text{\large Frame property}\\
\hline
(K\diamond 0)\quad K\diamond\phi=Cn(K\diamond\phi)&\vline&\text{No additional property}\\
\hline
(K\diamond 1)\quad\phi\in K\diamond\phi &\vline&(P^{*2}_{\diamond 1})\quad\begin{array}{l}
\forall s\in S, \forall E\in 2^S\setminus\varnothing, \\[-10pt]
 \bigcup\limits_{s'\in\mathcal B(s)} {f(s',E)}\subseteq E\end{array}\\
\hline
(K\diamond 2)\quad \text{If }\phi\in K\text{ then }K\diamond\phi=K &\vline&(P\diamond 2) \quad\begin{array}{l}
\forall s\in S,\forall E\in 2^S\setminus\varnothing,\\[-10pt]\text{ if } \mathcal B(s)\subseteq E\text{ then, }\\[-10pt]\bigcup\limits_{s'\in\mathcal B(s)} {f(s',E)}\,=\,\mathcal B(s)\end{array}\\
\hline
(K\diamond 3b)\quad\begin{array}{l}\text{If }\neg\phi\text{ is not a tautology}\\[-10pt]
\text{then }K\diamond\phi\ne\Phi_0
\end{array}&\vline&\begin{array}{ll}(P^{*5b}_{\diamond 3b})& \forall s\in S,\forall E\in 2^S\setminus\varnothing, \\[-10pt]
{}&\exists s'\in\mathcal B(s)\text{ such that } f(s',E)\ne\varnothing\end{array}\\
\hline
(K\diamond 4)\quad\begin{array}{l}\text{if }\phi\leftrightarrow\psi \text{ is a tautology}\\[-10pt]
\text{then }K\diamond\phi=K\diamond\psi\end{array}&\vline&\text{No additional property}\\
\hline
(K\diamond 5)\quad K\diamond(\phi\wedge\psi)\subseteq (K\diamond\phi)+\psi&\vline&\begin{array}{l}
(P^{*7}_{\diamond 5})\quad \forall s\in S,\forall E,F\in 2^S\text{ with }E\cap F\ne\varnothing,\\[-10pt]\bigcup\limits_{s'\in\mathcal B(s)} {\left(f(s',E)\cap F\right)}\,\,\subseteq\,\,\bigcup\limits_{s'\in\mathcal B(s)} {f(s',E\cap F)} 
\medskip\end{array}\\
\hline
(K\diamond 6w)\quad\begin{array}{l}\text{If }\psi\in K\diamond\phi \text{ and } \phi\in K\diamond\psi\\[-8pt]
\text{ and }\top\in K\diamond(\phi\wedge\psi)\\[-8pt] \text{ then }K\diamond\phi=K\diamond\psi\end{array}&\vline&
\begin{array}{l}(P\diamond 6w)\quad
\forall s\in S,\forall E,F\in 2^S\text{ with }E\cap F\ne\varnothing\\[-6pt]
 \text{if } \bigcup\limits_{s'\in\mathcal{B}(s)} {f(s',E)}\subseteq F\text{ and}
 \bigcup\limits_{s'\in\mathcal B(s)} {f(s',F)}\subseteq E\\[-4pt]
 \text{then }\bigcup\limits_{s'\in\mathcal{B}(s)} {f(s',E)}\,\,=\,\bigcup\limits_{s'\in\mathcal{B}(s)} {f(s',F)}\medskip
 \end{array}
 \\
\hline
(K\diamond 7s)\quad(K\diamond\phi)\cap(K\diamond\psi)\subseteq K\diamond(\phi\vee\psi)&\vline&\begin{array}{l}(P\diamond 7s)\qquad
\forall s\in S,\forall E,F\in 2^S\setminus\varnothing\\[-6pt]
\bigcup\limits_{s'\in\mathcal B(s)} {f(s',E\cup F)}\\[8pt]
 \subseteq \left(\bigcup\limits_{s'\in\mathcal B(s)} {f(s',E)}\right)\,\,\cup \,\left(\bigcup\limits_{s'\in\mathcal B(s)} {f(s',F)}\right)
 \end{array}
\end{array}$
}
\end{center}
\caption{Semantic characterization of the KM axioms.}
\label{Fig_1}
\end{figure}

The table in Figure \ref{Fig_1} lists, for every KM axiom, the characterizing property of frames.\footnote{
Some of the properties in Figure \ref{Fig_1} are related to properties of the selection function discussed in the literature on conditionals (\cite{Lew71}). For example, $(P^{\ast2}_{\diamond1})$ is related to Identity  and $(P^{\ast5b}_{\diamond3b})$ to Normality. The remaining properties do not seem to be related to properties discussed in that literature.
}\ 
The proofs are given in the Appendix. 

The reason for the absence of axiom $(K\diamond 3a)$ from the table in Figure \ref{Fig_1} will become clear in the next section. When a KM axiom coincides with an AGM axiom, the name of the corresponding semantic property reflects this; for example, since KM axiom $(K\diamond1)$ coincides with AGM axiom $(K\ast2)$, the corresponding property is denoted by $(P^{*2}_{\diamond 1})$.\footnote{The AGM axioms are listed in Section \ref{SEC:compare}.} 

\section{A modal logic for belief update}
We now turn to the modal language considered in \cite{Bon25b}, which contains three modal operators: a unimodal belief operator $B$, a bimodal conditional operator $>$ and the unimodal necessity operator $\square$. The interpretation of $B\phi$ is "the agent believes $\phi$",  the interpretation of $\phi>\psi$ is "if $\phi$ is (or were) the case then $\psi$ is (or would be) the case" and the interpretation of $\square\phi$ is "$\phi$ is necessarily true".
\par\noindent
The set $\Phi$ of formulas in the language is defined as follows: 
\begin{itemize}
\item $\Phi_0\subseteq \Phi$ (recall that $\Phi_0$ is the set of Boolean formulas built on the countable set of atomic sentences \texttt{At}),
\item if $\alpha,\beta\in\Phi$ then all of the following belong to $\Phi$: $\square\alpha$, $B\alpha$, $\alpha>\beta$ and all their Boolean combinations.
    \end{itemize}
We focus on the basic normal logic, denoted by $\mathcal L$, consisting of the following axioms and rules of inference.\footnote{We follow the nomenclature in \cite[p.115]{Che84}.} We denote general formulas by $\alpha$, $\beta$ and $\gamma$, while $\phi$, $\psi$ and $\chi$ are reserved for \emph{Boolean} formulas (e.g. in Figures 2-4). 
\begin{itemize}
\item Every formula that has the form of a classical tautology is a theorem.
\item The consistency axiom $D$ for $B$: 
\[\qquad(D_B)\qquad B\alpha\rightarrow\neg B\neg\alpha.\]
\item The conjunction axiom $C$ for $\square$, $B$ and $>$:
\[\arraycolsep=10pt\renewcommand{\arraystretch}{2}\begin{array}{ll}
(C_\square)&\square\alpha\wedge\square\beta\,\rightarrow\,\square(\alpha\wedge\beta)\\
(C_B)&B\alpha\wedge B\beta\,\rightarrow\, B(\alpha\wedge\beta)\\
(C_>)&(\gamma>\alpha)\wedge(\gamma>\beta)\,\rightarrow\,(\gamma>(\alpha\wedge\beta))\\
\end{array}\]
\item The necessity-to-belief axiom:\quad$(NB)\quad \square\alpha\rightarrow B\alpha$
\item The rule of inference\emph{ Modus Ponens}:\quad$(MP)\quad\dfrac{\alpha\,,\,\alpha\rightarrow\beta}{\beta}$ 
\item The rule of inference \emph{Necessitation} for  $\square$ and $>$:\\[3pt]
\[(N_\square)\quad\dfrac{\alpha}{\square\alpha} \qquad\qquad\qquad (N_>)\quad\dfrac{\beta}{\alpha>\beta}\]
\item The rule of inference $RM$ for  $\square$, $B$ and $>$:
\[\begin{array}{llll}
(RM_\square)&\dfrac{\alpha\rightarrow\beta}{\square\alpha\rightarrow\square\beta}&
(RM_B)&\dfrac{\alpha\rightarrow\beta}{B\alpha\rightarrow B\beta}\\[16pt]
(RM_>)&\dfrac{\alpha\rightarrow\beta}{(\gamma>\alpha)\rightarrow(\gamma>\beta)}&&
\end{array}\]
\end{itemize}
\begin{remark}
\label{REM:derivedAxRul}
(A) The following (which will be used in the proofs) are theorems of logic $\mathcal{L}$:
\begin{equation*}
\arraycolsep=10pt\renewcommand{\arraystretch}{2.8}
\begin{array}{ll}
(C_{\neg\square\neg})&\neg\square\neg(\alpha\wedge\beta)\,\rightarrow\,\neg\square\neg\alpha\\
(C_B^{inv})&B(\alpha\wedge\beta)\,\rightarrow\,B\alpha\wedge B\beta \\
(K_>)&(\alpha>\beta)\wedge (\alpha>(\beta\rightarrow\gamma))\,\rightarrow\, (\alpha>\gamma)\\
(A^{*1}_{\diamond 0})&B(\alpha>\beta)\wedge B(\alpha>(\beta\rightarrow\gamma))\,\rightarrow\, B(\alpha>\gamma)
\end{array}
\end{equation*}
(B) The following are derived rules of inference in logic $\mathcal L$:
\begin{equation*}
\arraycolsep=12pt
\begin{array}{ll}
(RM_{\neg\square\neg})& \dfrac{\alpha\rightarrow\beta}{\neg\square\neg\alpha\rightarrow\neg\square\neg\beta}\\[18pt]
(N_B)&\dfrac{\alpha}{B\alpha}\\[18pt]
(RM_{B>})&\dfrac{\alpha\rightarrow\beta}{B(\gamma>\alpha)\rightarrow B(\gamma>\beta)}
\end{array}
\end{equation*}
\end{remark}
\subsection{Frame correspondence}
As semantics for this modal logic we take the Kripke-Lewis frames of Definition \ref{DEF:frame}. A model based on a frame is obtained, as before, by adding a valuation $V:\texttt{At}\to 2^S$. The following definition expands Definition \ref{Truth0} by adding validation rules for formulas of the form $\square\alpha$, $\alpha>\beta$ and $B\alpha$.
\begin{definition}
\label{TruthModal}
Truth of a formula $\alpha\in\Phi$ at state $s$ in model $M$ (denoted by $s\models_M\alpha$) is defined as follows:
\begin{enumerate}
\item if $p\in\texttt{At}$ then $s\models_M p$ if and only if $s\in V(p)$.
\item $s\models_M\neg\alpha$ if and only if $s\not\models_M\alpha$.
\item $s\models_M(\alpha\vee\beta)$ if and only if $s\models_M\alpha$ or $s\models_M\beta$ (or both).
\item $s\models_M\square\alpha$ if and only if, $\forall s'\in S$, $s'\models_M\alpha$ (thus $s\models_M\neg\square\neg\alpha$ if and only if, for some $s'\in S$, $s'\models_M\alpha$, that is, $\Vert\alpha\Vert_M\ne\varnothing$).
\item $s\models_M(\alpha>\beta)$ if and only if, either
\begin{enumerate}[label=(\alph*)]
\item $s\models_M\square\neg\alpha$ (that is, $\Vert\alpha\Vert_M=\varnothing$), or
\item $s\models_M\neg\square\neg\alpha$ (that is, $\Vert\alpha\Vert_M\ne\varnothing$) and, for every $s'\in f(s,\Vert\alpha\Vert_M)$, $s'\models_M\beta$ (that is, $ f(s,\Vert\alpha\Vert_M)\subseteq\Vert\beta\Vert_M$).\footnote{Recall that, by definition of frame, $f(s,E)$ is defined only if $E\ne\varnothing$.}
\end{enumerate}
\item $s\models_M B\alpha$ if and only if, $\forall s'\in\mathcal B(s)$, $s'\models_M\alpha$ (that is, $\mathcal B(s)\subseteq \Vert\alpha\Vert_M$).
\end{enumerate}
\end{definition}
The definitions of validity and characterization are as in the previous section.
\begin{definition}
\label{DEF:validity2}
A formula $\alpha\in\Phi$ is \emph{valid on a frame $F$} if, for every model $M$ based on that frame and for every state $s$ in that model, $s\models_M\alpha$. A formula $\alpha\in\Phi$ is \emph{valid on a set of frames $\mathcal F$} if it is valid on every frame $F\in\mathcal F$.
\end{definition}
\begin{definition}
\label{DEF:MODALcorrespondence}
A formula $\alpha\in\Phi$ \emph{ is characterized by}, or \emph{corresponds to}, or \emph{characterizes}, a property $P$ of frames if the following is true:
\begin{enumerate}[label=(\arabic*)]
\item $\alpha$ is valid on the class of frames that satisfy property $P$, and
\item if a frame does not satisfy property $P$ then $\alpha$ is not valid on that frame.
    \end{enumerate}
\end{definition}

\par
The table in Figure \ref{Fig_2} shows, for every property of frames considered in Figure \ref{Fig_1}, the modal formula that corresponds to it. When a KM axiom coincides with an AGM axiom, the name of the corresponding modal formula reflects this; for example, since KM axiom $(K\diamond1)$ coincides with AGM axiom $(K\ast2)$, the corresponding modal formula is denoted by $(A^{*2}_{\diamond 1})$.\footnote{The AGM axioms are listed in Section \ref{SEC:compare}.} 
\bigskip\par
The proofs of the characterizations results are given in the Appendix.\bigskip
\par\noindent 
In axioms $(A^{*5b}_{\diamond3b})$ and $(A\diamond7s)$ the clause $\neg\square\neg\phi$ in the antecedent ensures that, on the semantic side, $\Vert\phi\Vert\ne\varnothing$; in particular, it rules out that $\phi$ is a contradiction; similarly for the clauses $\neg\square\neg\psi$ and $\neg\square\neg(\phi\wedge\psi)$. \,
The interpretation of the axioms is as follows:\\[6pt]
$\arraycolsep=10pt
\begin{array}{ll}
(A^{*2}_{\diamond1})&\text{the agent believes that if }\phi\text{ were the case, then }\phi\text{ would be}\\
{}&\text{the case}\\[5pt]
(A\diamond2)&\text{if the agent initially believes that }\phi \text{ then she believes that }\psi\\
{}&\text{if and only if she believes that if }\phi\text{ were the case, then }\psi\\
{}&\text{ would be the case}\\[5pt]
(A^{*5b}_{\diamond3b})&\text{if }\phi \text{ is not necessarily false and the agent believes that if }\phi \text{ were}\\
{}&\text{the case, then }\psi\text{ would be the case, then the agent does not}\\
{}&\text{believe that if }\phi\text{ were the case, then }\psi\text{ would not be the case}\\[5pt]
(A^{*7}_{\diamond5})&\text{if the conjunction }\phi\wedge\psi \text{ is not necessarily false and the agent}\\
{}&\text{believes that if }\phi\wedge\psi\text{ were the case, then }\chi\text{ would be the case}\\
{}&\text{then she believes that if }\phi\text{ were the case, then it would be the}\\
{}&\text{case that either }\psi \text{ is false or }\chi\text{ is true (that is, that }\psi\rightarrow\chi)\\[5pt]
(A\diamond6)&\text{if }\phi\wedge\psi \text{ is not necessarily false and the agent believes that}\\
{}&\text{if }\phi\text{ were the case, then }\psi\text{ would be the case and that if }\psi\text{ were}\\
{}&\text{the case, then }\phi\text{ would be the case, then she believes that if}\\
{}&\phi\text{ were the case, then }\chi\text{ would be the case if and only if she}\\
{}&\text{believes that if }\psi\text{ were the case, then }\chi\text{ would be the case}\\[5pt]
(A\diamond7s)&\text{if neither }\phi\text{ nor }\psi \text{ is necessarily false, and the agent believes}\\
{}&\text{that if }\phi\text{ were the case, then }\chi\text{ would be the case and that}\\
{}&\text{if }\psi\text{ were the case, then }\chi\text{ would be the case, then she}\\
{}&\text{believes that if }\phi\vee\psi\text{ were the case, then }\chi\text{ would be the case}\bigskip
\end{array}$
\par
Figure \ref{Fig_3} puts together Figures \ref{Fig_1} and \ref{Fig_2} by showing for each KM axiom the corresponding modal axiom or rule of inference. Note that $(A^{*1}_{\diamond0})$ is a theorem of $\mathcal L$ (see Remark \ref{REM:derivedAxRul}).\bigskip
\par
We denote by $\mathcal L_{KM}$ the extension of logic $\mathcal L$ obtained by adding the modal axioms and rules of inference listed in Figure \ref{Fig_3}. In the next section we define a similar extension of $\mathcal L$, denoted by $\mathcal L_{AGM}$, that captures the logic of AGM belief revision and show that  $\mathcal L_{KM}$ is contained in $\mathcal L_{AGM}$.
\begin{figure}[H]
\begin{center}
\small{
$\renewcommand{\arraystretch}{2}\arraycolsep=4pt
\begin{array}{*{20}{l}}
\qquad\text{\large Frame property}&{}&\,\,\text{\large Corresponding modal formula}\\[-12pt]
{}&{}&\text{In all the formulas, }\phi,\psi,\chi \text{ are Boolean}\\
\hline
(P^{*2}_{\diamond 1})\quad\begin{array}{l}
\forall s\in S, \forall E\in 2^S\setminus\varnothing, \\[-10pt]
  \bigcup\limits_{s'\in\mathcal B(s)} {f(s',E)}\subseteq E\end{array}&\vline&(A^{*2}_{\diamond1})\qquad B(\phi>\phi)\\
\hline
(P\diamond2)\,\,\begin{array}{l}
\forall s\in S, \forall E\in 2^S\setminus\varnothing, \\[-10pt]
\text{if } \mathcal B(s)\subseteq E\text{ then }\\[-8pt]
\bigcup\limits_{s'\in\mathcal B(s)} {f(s',E)}=\mathcal B(s)\end{array}&\vline&(A\diamond 2)\qquad B\phi\,\rightarrow\,\big(B\psi\leftrightarrow B(\phi>\psi)\big)\\
\hline
(P^{*5b}_{\diamond3b})\,\,\begin{array}{l}\forall s\in S,\forall E\in 2^S\setminus \varnothing,\\[-10pt]\exists s'\in\mathcal B(s) \text{ such that }f(s',E)\ne\varnothing\end{array}&\vline&\begin{array}{l}(A^{*5b}_{\diamond 3b})\\[-5pt]
\left(\neg\square\neg\phi\wedge B(\phi>\psi)\right)\rightarrow\neg B(\phi>\neg\psi)
\end{array}\\
\hline
\begin{array}{l}(P^{*7}_{\diamond 5})\\[-5pt]
\forall s\in S,\forall E,F\in 2^S\text{ with }E\cap F\ne\varnothing,\\[-10pt]\bigcup\limits_{s'\in\mathcal B(s)} {\left(f(s',E)\cap F\right)}\,\,\subseteq\,\,\bigcup\limits_{s'\in\mathcal B(s)} {f(s',E\cap F)} 
\medskip\end{array}&\vline&\begin{array}{l}(A^{*7}_{\diamond5})\\[-5pt]
\neg\square\neg(\phi\wedge\psi)\,\wedge\,B((\phi\wedge\psi)>\chi)\\[-5pt]
\rightarrow B\big(\phi>(\psi\rightarrow\chi)\big)\end{array}\\
\hline
\begin{array}{l}(P\diamond 6w)\\[-5pt]
\forall s\in S,\forall E,F\in 2^S\text{ with }E\cap F\ne\varnothing\\[-6pt]
 \text{if } \bigcup\limits_{s'\in\mathcal{B}(s)} {f(s',E)}\subseteq F\text{ and}
 \bigcup\limits_{s'\in\mathcal B(s)} {f(s',F)}\subseteq E\\[-4pt]
 \text{then }\bigcup\limits_{s'\in\mathcal{B}(s)} {f(s',E)}\,\,=\,\bigcup\limits_{s'\in\mathcal{B}(s)} {f(s',F)}\medskip
 \end{array}&\vline&\begin{array}{l}(A\diamond 6w)\\[-5pt]
\neg\square\neg(\phi \wedge\psi)\wedge B(\phi>\psi)\wedge B(\psi>\phi)\\[-5pt]
\rightarrow\left(B(\phi>\chi)\leftrightarrow B(\psi>\chi)\right)
 \end{array}\\
\hline
\begin{array}{l}(P\diamond 7s)\\[-5pt]
\forall s\in S,\forall E,F\in 2^S\setminus\varnothing\\[-6pt]
\bigcup\limits_{s'\in\mathcal B(s)} {f(s',E\cup F)}\\[8pt]
 \subseteq \left(\bigcup\limits_{s'\in\mathcal B(s)} {f(s',E)}\right)\,\,\cup \,\left(\bigcup\limits_{s'\in\mathcal B(s)} {f(s',F)}\right)
 \end{array} &\vline&\begin{array}{l}(A\diamond7s)\\[-5pt]
\neg\square\neg\phi \wedge \neg\square\neg\psi\wedge B(\phi>\chi)\wedge B(\psi>\chi)\\[-5pt]
\rightarrow B\left((\phi\vee\psi)>\chi\right)
\medskip
\end{array}\\
\hline
\end{array}$
}
\end{center}
\caption{The frame properties of Figure \ref{Fig_1} and the corresponding modal formulas}
\label{Fig_2}
\end{figure}
\begin{figure}[H]
\begin{center}
\small{
$\renewcommand{\arraystretch}{2}\arraycolsep=4pt
\begin{array}{*{20}{l}}
\qquad\text{\large KM axiom}&{}&\qquad\text{\large Modal axiom/Rule of inference}\\[-8pt]
{}&{}&\qquad\text{In all the formulas, }\phi,\psi,\chi \text{ are Boolean}\\
\hline
(K\diamond0)\,\,\,K\diamond \phi=Cn(K\diamond \phi)&\vline&\begin{array}{l}(A^{*1}_{\diamond0})\\[-5pt] B(\phi>\psi)\wedge B(\phi>(\psi\rightarrow\chi))\,\rightarrow\, B(\phi>\chi)\end{array}\\
\hline
(K\diamond1)\,\,\,\phi\in K\diamond \phi &\vline&(A^{*2}_{\diamond1})\qquad B(\phi>\phi)\\
\hline
(K\diamond2)\,\,\,\text{If }\phi\in K \text{ then }K\diamond\phi= K&\vline&(A\diamond 2)\qquad B\phi\,\rightarrow\,\big(B\psi\leftrightarrow B(\phi>\psi)\big)\\
\hline
(K\diamond3a)\,\,\begin{array}{l}\text{If }\neg\phi\text{ is a tautology}\\[-10pt]
\text{then }K\diamond\phi=\Phi_0
\end{array}&\vline&\,\,\begin{array}{l}\text{Rule of inference }\\[-5pt]
(R^{*5a}_{\diamond3a})\qquad\dfrac{\neg\phi}{B(\phi>\psi)}\medskip\end{array}\\
\hline
(K\diamond3b)\,\,\begin{array}{l}\text{If }\neg\phi\text{ is not a tautology}\\[-10pt]
\text{then }K\diamond\phi\ne\Phi_0
\end{array}&\vline&\,\,\begin{array}{l}(A^{*5b}_{\diamond 3b})\\[-5pt]
\left(\neg\square\neg\phi\wedge B(\phi>\psi)\right)\rightarrow\neg B(\phi>\neg\psi)
\end{array}\\
\hline
(K\diamond4)\,\,\begin{array}{l}\text{if }\phi\leftrightarrow\psi \text{ is a tautology}\\[-10pt]
\text{then }K\diamond\phi=K\diamond\psi\end{array}&\vline&\begin{array}{l}\text{Rule of inference }\\[-5pt]
(R^{*6}_{\diamond4})\qquad\dfrac{\phi\leftrightarrow\psi}{B(\phi>\chi)\leftrightarrow B(\psi>\chi)}\medskip\end{array}\\
\hline
(K\diamond5)\,\,\,K\diamond(\phi\wedge\psi)\subseteq (K\diamond\phi)+\psi&\vline&\begin{array}{l}(A^{*7}_{\diamond5})\\[-5pt]
\neg\square\neg(\phi\wedge\psi)\,\wedge\,B((\phi\wedge\psi)>\chi)\\[-5pt]
\rightarrow B\big(\phi>(\psi\rightarrow\chi)\big)\end{array}\\
\hline
(K\diamond 6w)\quad\begin{array}{l}\text{If }\psi\in K\diamond\phi \text{ and } \phi\in K\diamond\psi\\[-8pt]
\text{ and }\top\in K\diamond(\phi\wedge\psi)\\[-8pt] 
\text{ then }K\diamond\phi=K\diamond\psi\end{array}&\vline&\begin{array}{ll}(A\diamond 6w)&\neg\square\neg(\phi \wedge\psi)\wedge B(\phi>\psi)\wedge B(\psi>\phi)\\[-5pt]
{}&\rightarrow\left(B(\phi>\chi)\leftrightarrow B(\psi>\chi)\right)
 \end{array}\\
\hline
(K\diamond7s)\,\,(K\diamond\phi)\cap(K\diamond\psi)\subseteq K\diamond(\phi\vee\psi) &\vline&\begin{array}{l}(A\diamond7s)\\[-5pt]
\neg\square\neg\phi \wedge \neg\square\neg\psi\wedge B(\phi>\chi)\wedge B(\psi>\chi)\\[-5pt]
\rightarrow B\left((\phi\vee\psi)>\chi\right)
\medskip
\end{array}\\
\hline
\end{array}$
}
\end{center}
\caption{The KM axioms and the corresponding modal axioms/rules of inference}
\label{Fig_3}
\end{figure}
\section{Relating KM logic to AGM logic}
\label{SEC:compare}
In \cite{Bon25b} every AGM axiom of belief revision was translated into a corresponding modal axiom in a way similar to what was done in the previous section for the KM axioms. Figure \ref{Fig_4} (which reproduces Figure 3 in \cite{Bon25b}, with the modal axioms renamed to match the names in this paper) shows the correspondence between AGM axioms and modal axioms.\footnote{\label{FT:aboutK*4}In \cite{Bon25a} $(K\ast 4)$ was given in a weaker form, namely if $\neg\phi\notin K$ then $K\subseteq K\ast\phi$. However, in the presence of $(K\ast1)$ and $(K\ast2)$, the two are equivalent. That the stronger version used in Figure \ref{Fig_4} implies the weaker version follows from the fact that $K\subseteq K+\{\phi\}$. To prove the converse, let $\psi\in K+\phi$. Then since $K$ is deductively closed, $(\phi\rightarrow\psi)\in K$ so that, by the weaker version, $(\phi\rightarrow\psi)\in K\ast\phi$. By $(K\ast2)$, $\phi\in K\ast\phi$ and by $(K\ast1)$ $K\ast\phi$ is deductively closed. Thus, since $\phi,(\phi\rightarrow\psi)\in K\ast\phi$ it follows that $\psi\in K\ast\phi$.}
\par
We denote by $\mathcal L_{AGM}$ the extension of logic $\mathcal L$ obtained by adding the modal axioms and rules of inference listed in Figure \ref{Fig_4}. According to the following proposition, which is proved in the Appendix, AGM belief revision can be viewed as a strengthening of KM belief update.
\begin{proposition}
\label{PROP:sublogic}
The logic $\mathcal L_{KM}$ of KM belief update is contained in the logic $\mathcal L_{AGM}$ of AGM belief revision, that is, every axiom of $\mathcal L_{KM}$ is a theorem of $\mathcal L_{AGM}$.
\end{proposition}
We conclude this section by considering a stronger version of belief update suggested by \cite{KatMen91}. Katsuno and Mendelzon obtain this stronger version by replacing their axioms $(U6)$ and $(U7)$ with a stronger axiom, which they call $(U9)$.\footnote{The authors then show that this stronger notion of belief update corresponds semantically to \emph{total} pre-orders on the set of possible worlds.}\ In our framework their axiom $(U9)$ can be translated as follows (see \cite{Ara25,Pep93,Pepetal96}):
\begin{equation*}
(K\diamond9)\quad \text{If } K \text{ is complete and }\neg\psi\notin K\diamond\phi \text{ then }(K\diamond\phi)+\psi\subseteq K\diamond (\phi\wedge\psi).
  \end{equation*}
\par
The following Lemma is proved in the Appendix.
\begin{lemma}
\label{LEM:K9s}
The following (seemingly stronger) version of $(K\diamond9)$ (obtained by dropping the clause `if $K$ is complete' from $(K\diamond9)$) follows from $(K\diamond0)$, $(K\diamond8)$ and $(K\diamond9)$:
\[(K\diamond9s)\quad \text{If }\neg\psi\notin K\diamond\phi \text{ then }(K\diamond\phi)+\psi\subseteq K\diamond (\phi\wedge\psi)\]
\end{lemma}
\newpage
\begin{figure}[H]
\begin{center}
$\renewcommand{\arraystretch}{2}\arraycolsep=4pt
\begin{array}{*{20}{l}}
\qquad\text{\large AGM axiom}&{}&\begin{array}{c}\text{\large Modal axiom/Rule of Inference}\\[-10pt]\,(\text{for }\phi,\psi,\chi\in\Phi_0)\end{array}\\
\hline
(K*1)\,\,\,K*\phi=Cn(K*\phi)&\vline&
\begin{array}{ll}
(A^{*1}_{\diamond0})& B(\phi>\psi)\wedge B(\phi>(\psi\rightarrow\chi))\\[-10pt]
{}&\rightarrow B(\phi>\chi)\end{array}\\
\hline
(K*2)\,\,\,\phi\in K*\phi &\vline&
(A^{*2}_{\diamond1})\qquad B(\phi>\phi)\\
\hline
(K*3)\,\,\,K*\phi\subseteq K+\phi&\vline&(A*3)\quad
\big(\neg\square\neg\phi\wedge B(\phi>\psi\big)\rightarrow B(\phi\rightarrow\psi)\\
\hline
(K*4)\,\,\,\begin{array}{*{20}{l}}\text{if }\neg\phi\notin K\\[-10pt]
\text{ then }K+\phi\subseteq K*\phi\end{array}&\vline&
(A*4)\quad\big(\neg B\neg\phi\wedge B(\phi\rightarrow\psi)\big)\rightarrow B(\phi>\psi)\\
\hline
(K*5a)\,\,\begin{array}{*{20}{l}}\text{If }\neg\phi\text{ is a tautology, then }\\[-10pt]K*\phi=\Phi_0\end{array}&\vline&\begin{array}{ll}
\text{rule of inference:}&{}\\[-10pt]
(R^{*5a}_{\diamond3a})&\dfrac{\neg\phi}{ B(\phi>\psi)}\end{array} \\
(K*5b)\,\,\begin{array}{*{20}{l}}\text{If }\neg\phi\text{ is not a tautology}\\[-10pt]
\text{then }K*\phi\ne\Phi_0\end{array}&\vline&
(A^{*5b}_{\diamond3b})\quad\big(\neg\square\neg\phi\wedge B(\phi>\psi)\big)\rightarrow \neg B(\phi>\neg\psi)\\
\hline
(K*6)\,\,\begin{array}{*{20}{l}}\text{if }\phi\leftrightarrow\psi \text{ is a tautology}\\[-10pt]
\text{then }K*\phi=K*\psi\end{array}&\vline&\begin{array}{ll}
\text{rule of inference:}\\[-10pt]
(R^{*6}_{\diamond4})&\dfrac{\phi\leftrightarrow \psi}{ B(\phi>\chi)\leftrightarrow B(\psi>\chi)}\medskip\end{array}\\
\hline
(K*7)\,\,\,K*(\phi\wedge\psi)\subseteq (K*\phi)+\psi&\vline&
\begin{array}{*{20}{ll}}(A*7)&\neg\square\neg(\phi\wedge\psi)\,\wedge\,B((\phi\wedge\psi)>\chi)\\[-10pt] {}&\rightarrow B\big(\phi>(\psi\rightarrow\chi)\big)\end{array}\\
\hline
(K*8)\,\,\begin{array}{*{20}{l}}\text{If }\neg\psi\notin K*\phi,\text{ then}\\[-10pt]
(K*\phi)+\psi\subseteq K*(\phi\wedge\psi)\end{array}&\vline&
\begin{array}{*{20}{ll}}(A^{*8}_{\diamond9s})&\neg B(\phi>\neg\psi)\wedge B(\phi>(\psi\rightarrow\chi))\\[-10pt]
{}&\rightarrow\, B\big((\phi\wedge\psi)>(\psi\wedge\chi)\big)\end{array}
\end{array}$
\end{center}
\caption{The correspondence between AGM axioms and their modal counterparts.}
\label{Fig_4}
\end{figure}
\pagebreak

\begin{definition}
\label{DEF:strong_update function}
A \emph{strong belief update function} is a full-domain belief change function $\diamond:\Phi_0\to 2^{\Phi_0}$ that satisfies axioms $(K\diamond0)$-$(K\diamond5)$ and $(K\diamond9s)$.
\end{definition}
\par\noindent
Comparing the AGM axioms to the axioms of the strong version of KM update we can see that:
\begin{itemize}
\item KM axiom $(K\diamond0$) coincides with AGM axiom $(K\ast1)$ and the corresponding modal axiom is the following, which is a theorem of $\mathcal L$ (see Remark \ref{REM:derivedAxRul})
\[(A^{*1}_{\diamond0})\quad 
B(\phi>\psi)\wedge B(\phi>(\psi\rightarrow\chi))\,\rightarrow\, B(\phi>\chi)\]
\item  KM axiom $(K\diamond1$) coincides with AGM axiom $(K\ast2)$ and the corresponding modal axiom is
\[(A^{*2}_{\diamond1})\quad B(\phi>\phi)\]
\item  KM axiom $(K\diamond3a$) coincides with AGM axiom $(K\ast5a)$ and it corresponds to the rule of inference
\[(R^{*5a}_{\diamond3a})\quad\dfrac{\neg\phi}{B(\phi>\psi)}\]
\item  KM axiom $(K\diamond3b$) coincides with AGM axiom $(K\ast5b)$ and the corresponding modal axiom is
\[(A^{*5b}_{\diamond3b})\quad \left(\neg\square\neg\phi\wedge B(\phi>\psi)\right)\rightarrow\neg B(\phi>\neg\psi)\]
\item  KM axiom $(K\diamond4$) coincides with AGM axiom $(K\ast6)$ and it corresponds to the rule of inference
\[(R^{*6}_{\diamond4})\quad\dfrac{\phi\leftrightarrow\psi}{B(\phi>\chi)\leftrightarrow B(\psi>\chi)}\]
\item  KM axiom $(K\diamond5$) coincides with AGM axiom $(K\ast7)$ and the corresponding modal axiom is
\[(A^{*7}_{\diamond5})\quad \neg\square\neg(\phi\wedge\psi)\,\wedge\,B((\phi\wedge\psi)>\chi)\,
\rightarrow\, B\big(\phi>(\psi\rightarrow\chi)\big)\]
\item  KM axiom $(K\diamond9s$) coincides with AGM axiom $(K\ast8)$ and the corresponding modal axiom is
\[(A^{*8}_{\diamond9s})\quad \neg B(\phi>\neg\psi)\wedge B(\phi>(\psi\rightarrow\chi))\,
\rightarrow\, B\big((\phi\wedge\psi)>(\psi\wedge\chi)\big)\]
\end{itemize}
Thus the difference between AGM belief revision and the strong version of KM belief update is that the former contains axioms $(K\ast3)$ and $(K\ast4)$ while the latter only requires axiom $(K\diamond2)$. 
First of all, note that axiom $(A\diamond2)$ -- corresponding to $(K\diamond2)$ -- is a theorem of $\mathcal L_{AGM}$ (see Lemma \ref{LEM:forKdiamond2} in the Appendix) and thus \emph{the logic $\mathcal L_{AGM}$ contains also the logic of the strong version of KM belief update}. The following lemma, proved in the Appendix, shows that the modal axiom corresponding to AGM axiom $(K\ast3)$ is provable in logic $\mathcal L_{KM}$.
\begin{lemma}
\label{LEM:K*3inKM}
Axiom 
\[(A\ast3)\qquad\big(\neg\square\neg\phi\wedge B(\phi>\psi\big)\rightarrow B(\phi\rightarrow\psi)\]
(which is the modal counterpart to AGM axiom $(K\ast3)$) is provable in logic $\mathcal L_{KM}$.
\end{lemma}
Thus, the difference between AGM belief revision and the strong version of KM belief update reduces to one axiom. Axiom \[(A\diamond2)\qquad B\phi\wedge\  B\psi\,\rightarrow\,B(\phi>\psi)\]
 in the latter is replaced, in the former, by the stronger axiom 
 \[(A\ast4)\qquad \neg B\neg\phi\wedge B(\phi\rightarrow\psi)\,\rightarrow\, B(\phi>\psi)\]
 Axiom $(A\ast4)$ covers the case where the  informational input $\phi$ is initially \emph{not disbelieved} ($\neg\phi\notin K$). Comparing the frame property $(P\diamond2)$ corresponding to KM axiom $(A\diamond2)$ (see Figure \ref{Fig_2}) with the following property, which is implied by the property corresponding to AGM axiom  $(A\ast4)$\footnote{For completeness, since the (standard) version of $(K\ast 4)$ used here is somewhat different from the one used in  \cite{Bon25a,Bon25b} (see Footnote \ref{FT:aboutK*4}) we prove the correspondence results for $(K\ast4)$ and $(A\ast4)$ in Lemma \ref{LEM:K*4} in the Appendix. Note that Property $(P\ast4)$ given in Lemma \ref{LEM:K*4} implies the above property. In fact, since $\mathcal B(s)=\left(\mathcal B(s)\cap(S\setminus E)\right)\cup \left(\mathcal B(s)\cap E\right)\subseteq (S\setminus E)\cup \left(\mathcal B(s)\cap E\right)$, taking $F$ in $(P\ast4)$ to be $\mathcal B(s)\cap E$ we get the above property.}
\[ \text{if }\mathcal B(s)\cap E\ne\varnothing\, \text{ then }\bigcup\limits_{s'\in\mathcal B(s)}{f(s',E)}\,\subseteq\, B(s)\cap E\]
we can see that the AGM theory requires that (letting $E=\Vert\phi\Vert$) the revised beliefs be concentrated on the $\phi$-states in $\mathcal B(s)$, while the strong version of KM update allows the revised beliefs to include $\phi$-states outside of $\mathcal B(s)$. 
\par\noindent
On the other hand, there is no difference between the two theories when the informational input $\phi$ is initially disbelieved, that is, when the agent initially believes $\neg\phi$. 

\section{Conclusion}
As remarked in \cite{Bon25b}, translating the AGM axioms of belief revision and the KM axioms of belief update into modal formulas has the advantage of unifying the treatment of belief, belief revision and belief update under the same umbrella. Starting with Hintikka's \cite{Hin62} seminal contribution, the notion of belief has been studied within the context of modal logic, which allows one to express properties such as positive introspection of beliefs ($B\phi\rightarrow BB\phi$), negative introspection of beliefs ($\neg B\phi\rightarrow B\neg B\phi$), the relationship between knowledge and belief, etc. 
\par 
The modal AGM axioms and KM axioms listed in Figures \ref{Fig_3} and \ref{Fig_4} are restricted axioms in that the formulas $\phi$, $\psi$ and $\chi$ that appear in them are required to be Boolean. Some of those axioms involve some nesting of the operators, such as $B(\phi>\psi)$. Logic $\mathcal L$ opens the door to investigating properties of belief change that go beyond those considered in the AGM theory and the KM theory. For example, one can investigate introspection properties for suppositional beliefs: $B(\phi>\psi)\rightarrow BB(\phi>\psi)$ and $\neg B(\phi>\psi)\rightarrow B\neg B(\phi>\psi)$. Other nestings of the operators that go beyond those considered in Figures \ref{Fig_3} and \ref{Fig_4} may also be worth considering, such as $B(B\phi> B\psi)$ (the agent believes that if she were to believe $\phi$ then she would believe $\psi$) or $B\big(\phi>(\psi>\chi)\big)$ (the agent believes that if $\phi$ were the case then if $\psi$ were the case then $\chi$ would be the case). Perhaps some of these more complex formulas will turn out to be useful in characterizing the notions of iterated belief revision or iterated belief update.
\par
The main result of this paper is that the modal logic $\mathcal L_{KM}$ of KM belief update is contained in the modal logic $\mathcal L_{AGM}$ of AGM belief revision, thus highlighting the fact that there is no conceptual difference between the two notions: one is merely a special case of the other. Furthermore, if one focuses on the strong version of KM belief update, then the difference between the two theories can be narrowed down to the different ways in which they deal with unsurprising information (that is, with formulas $\phi$ that were not initially disbelieved), while there is no difference in how the two treat surprising information, that is, information that contradicts the initial beliefs.

\appendix
\section{Proofs}
\textbf{Lemma \ref{LEM:K7s}.}\quad \emph{$(K\diamond 7s)$ follows from $(K\diamond 7)$ and $(K\diamond 8)$}.\\
\begin{proof}
By $(K\diamond 8)$ (recall the assumption that $K$ is consistent so that $\llbracket K\rrbracket\ne\varnothing$), 
\begin{equation}
\label{ToK7sa}
\begin{array}{lll}K\diamond\phi \,\cap\, K\diamond\psi &=& \left(\bigcap\limits_{w\in\llbracket K\rrbracket}{\left( w \diamond \phi \right)}\right) \,\cap\, \left(\bigcap\limits_{w\in\llbracket K\rrbracket} {\left( {w \diamond \psi} \right)}\right)\\[24pt]
{}&=&\bigcap\limits_{w\in\llbracket K\rrbracket}{\big(\left(w\diamond \phi\right)\,\cap\, \left( w\diamond \psi\right)\big)}
\end{array} 
\end{equation}
By $(K\diamond 7)$ (since every MCS is complete), for every $w\in\llbracket K\rrbracket$, $\left(w\diamond \phi\right)\,\cap\, \left( w\diamond \psi\right)\,\subseteq\, w \diamond (\phi\vee\psi)$; thus,
\begin{equation}
\label{ToK7sb}
\bigcap\limits_{w\in\llbracket K\rrbracket}{\big(\left(w\diamond \phi\right)\,\cap\, \left( w\diamond \psi\right)\big)}\,\subseteq\, \bigcap\limits_{w\in\llbracket K\rrbracket}{w \diamond (\phi\vee\psi)}\end{equation}
It follows from \eqref{ToK7sa} and \eqref{ToK7sb} that
\begin{equation}
\label{ToK7sc}
\left(K\diamond\phi \cap K\diamond\psi\right)\,\subseteq\,\bigcap\limits_{w\in\llbracket K\rrbracket} {w\diamond(\phi\vee\psi)}  
\end{equation}
and by $(K\diamond 8)$
\begin{equation}
\label{ToK7sd}
\bigcap\limits_{w\in\llbracket K\rrbracket} {w\diamond(\phi\vee\psi)}\ =\,K\diamond(\phi\vee\psi)
\end{equation}
Thus, by \eqref{ToK7sc} and \eqref{ToK7sd},
\[K\diamond\phi \,\cap\, K\diamond\psi\,\subseteq\,K\diamond(\phi\vee\psi)\]
\end{proof}
\pagebreak\noindent
\textbf{Proofs of the semantic characterizations listed in Figure \ref{Fig_1}}.
\par\noindent
The correspondence for 
\begin{itemize}
\item[-] $(K\diamond1)$ is proved in \cite[p.4]{Bon25b} (since $(K\diamond1)$ coincides with AGM axiom $(K\ast2)$).
\item[-] $(K\diamond3b)$ is proved in \cite[p.4]{Bon25b} (since $(K\diamond3b)$ coincides with AGM axiom $(K\ast5b)$).
\item[-] $(K\diamond5)$ is proved in Proposition 3 in \cite{Bon25a}.\footnote{The characterization there was proved for a differently worded property, namely\linebreak 
    $\forall G\in 2^S,\text{ if }\bigcup\limits_{s'\in\mathcal B(s)}{f(s',E\cap F)}\subseteq G\text{ then } \bigcup\limits_{s'\in\mathcal B(s)}{(f(s',E)\cap F)}\subseteq G$. It is straightforward to verify that this property is equivalent to property $(P^{*7}_{\diamond5})$ in Figure \ref{Fig_1}.}
\end{itemize}
Thus we only need to prove the characterization for  $(K\diamond2)$, $(K\diamond6w)$ and $(K\diamond7s)$, which is done in the following three propositions.\footnote{A weaker property than $(P\diamond2)$ (where in the consequent '$\subseteq$' was used instead of '$=$') was sufficient in \cite{Bon25a} to characterize $(K\diamond2)$ because in that paper the definition of frame, unlike in this paper, included the property of Weak Centering: if $s\in E$ then $s\in f(s,E)$.}
\begin{proposition}
Frame property\\[5pt]
$(P\diamond2)\quad\forall s\in S, \forall E\in 2^S\setminus\varnothing,\,
\text{if } \mathcal B(s)\subseteq E\text{ then} 
\bigcup\limits_{s'\in\mathcal B(s)} {f(s',E)}\,=\,\mathcal B(s)$\\[5pt]
characterizes the following KM axiom:\\[5pt]
$(K\diamond2)\quad \text{if }\phi\in K \text{ then } K\diamond \phi=K$.
\end{proposition}
\begin{proof}
(A) Fix an arbitrary model based on a frame that satisfies Property $(P\diamond 2)$, an arbitrary state $s\in S$ and let $\diamond$ be the partial belief change function based on $K_s$  defined by ($\ref{RI}$). Let $\phi\in\Phi_0$ be such that $\phi\in K_s$, that is, $\mathcal B(s)\subseteq \Vert\phi\Vert$ (thus, by seriality of $\mathcal B$, $\Vert\phi\Vert\ne\varnothing$). We need to show that, for every formula $\psi\in\Phi_0$, $\psi\in K\diamond\phi$ if and only if $\psi\in K$, that is, $\bigcup\limits_{s'\in\mathcal B(s)}{f(s',\Vert\phi\Vert)}\subseteq\Vert\psi\Vert$ if and only if  $\mathcal B(s)\subseteq\Vert\psi\Vert$ . This is an immediate consequence of $(P\diamond 2)$ with $E=\Vert\phi\Vert)$. \smallskip 
\par\noindent
 (B) Conversely, fix a frame that violates Property $(P\diamond 2)$. Then there exist $s\in S$  and $E\in 2^S$  such that $\mathcal B(s)\subseteq E$ (thus, by seriality of $\mathcal B$, $E\ne\varnothing$) but $\bigcup\limits_{s'\in\mathcal B(s)} {f(s',E)}\ne\mathcal B(s)$. Two cases are possible.\\[5pt]
CASE 1: $\bigcup\limits_{s'\in\mathcal B(s)}{f(s',E)}\not\subseteq\mathcal B(s)$. Let $p,q\in\texttt{At}$ be atomic formulas and construct a model where $\Vert p\Vert=E$ and $\Vert q\Vert=\mathcal B(s)$. Then, since $\mathcal B(s)\subseteq E=\Vert p\Vert$, $p\in K_s$ and, since $\mathcal B(s)\subseteq\Vert q\Vert$,  $q\in K_s$ but, since  $\bigcup\limits_{s'\in\mathcal B(s)}{f(s',\Vert p\Vert)}\not\subseteq\Vert q\Vert$, $q\notin K_s\diamond p$. Hence $K_s\diamond p\ne K_s$.\\[5pt]
CASE 2: $\mathcal B(s)\not\subseteq\bigcup\limits_{s'\in\mathcal B(s)} {f(s',\Vert\ p \Vert)}$. Let $p,q\in\texttt{At}$ be atomic formulas and construct a model where $\Vert p\Vert=E$ and $\Vert q\Vert=\bigcup\limits_{s'\in\mathcal B(s)} {f(s',E)}$. Then, $p\in K_s$ and $q\in K_s\diamond p$ but $q\notin K_s$, so that $K_s\diamond p\ne K_s$.\medskip
\end{proof}
\begin{proposition}
Frame property
\[(P\diamond6w)\quad\begin{array}{l}\forall s\in S,\forall E,F\in 2^S\text{ with }E\cap F\ne\varnothing\\[8pt]
 \text{if } \bigcup\limits_{s'\in\mathcal{B}(s)} {f(s',E)}\subseteq F\text{ and}
 \bigcup\limits_{s'\in\mathcal B(s)} {f(s',F)}\subseteq E\\[16pt]
 \text{then }\bigcup\limits_{s'\in\mathcal{B}(s)} {f(s',E)}\,\,=\,\bigcup\limits_{s'\in\mathcal{B}(s)} {f(s',F)}
 \end{array}\]
characterizes the following KM axiom:
\[(K\diamond6w)\quad \text{If }\psi\in K\diamond\phi\text{ and }\phi\in K\diamond\psi\text{ and }\top\in K\diamond(\phi\wedge\psi)\text{ then } K\diamond\phi=K\diamond\psi\]
\end{proposition}
\begin{proof}
Fix a frame that satisfies property $(P\diamond6w)$, an arbitrary model $M$ based on it and an arbitrary state $s\in S$ and let $\diamond$ be the belief change function defined by \eqref{RI}.  Let $\phi,\psi,(\phi\wedge\psi)\in\Phi_0$ be in the domain of $\diamond$ (hence  $\Vert\phi\wedge\psi\Vert\ne\varnothing$ and thus $\Vert\phi\Vert\ne\varnothing$ and $\Vert\psi\Vert\ne\varnothing$) and suppose that $\psi\in K_s\diamond\phi$ and $\phi\in K_s\diamond\psi$ (note that, by Remark \ref{REM:topinKdiamond}, $\top\in K\diamond(\phi\wedge\psi)$ ). Then, $\bigcup\limits_{s'\in\mathcal B(s)} {f(s',\Vert\phi\Vert)}\subseteq\Vert\psi\Vert$ and $\bigcup\limits_{s'\in\mathcal B(s)} {f(s',\Vert\psi\Vert)}\subseteq\Vert\phi\Vert$ and thus, by Property $(P\diamond6w)$ (with $E=\Vert \phi\Vert$ and $F=\Vert\psi\Vert$; note that $\Vert\phi\Vert\cap\Vert\psi\Vert=\Vert\phi\wedge\psi\Vert\ne\varnothing$)
\begin{equation}
\label{EQ:due}
\bigcup\limits_{s'\in\mathcal B(s)} {f(s',\Vert\phi\Vert)} =\bigcup\limits_{s'\in\mathcal B(s)} {f(s',\Vert\psi\Vert)}
\end{equation}
It follows from \eqref{EQ:due} that,  for every $\chi\in\Phi_0$, $\chi\in K_s\diamond\phi$ if and only if $\chi\in K_s\diamond\psi$, that is, $K_s\diamond\phi=K_s\diamond\psi$\\[3pt]
Conversely, fix a frame that violates property $(P\diamond6w)$. Then there exist $s\in S$ and $E,F\in 2^S$ such that $E\cap F\ne\varnothing$,  $\bigcup\limits_{s'\in\mathcal B(s)} {f(s',E)}\subseteq F$ and $\bigcup\limits_{s'\in\mathcal B(s)} {f(s',F)}\subseteq E$ but 
$\bigcup\limits_{s'\in\mathcal B(s)} {f(s',E)}\ne \bigcup\limits_{s'\in\mathcal B(s)} {f(s',F)}$
Two cases are possible:\\[5pt]
CASE 1: $\bigcup\limits_{s'\in\mathcal B(s)} {f(s',E)} \,\nsubseteq\,\bigcup\limits_{s'\in\mathcal B(s)} {f(s',F)}$.\\[5pt]
CASE 2: $\bigcup\limits_{s'\in\mathcal B(s)} {f(s',F)} \,\nsubseteq\,\bigcup\limits_{s'\in\mathcal B(s)} {f(s',E)}$.\\[5pt]
In Case 1, let $p,q,r\in\texttt{At}$ be atomic sentences and construct a model based on this frame where $\Vert p\Vert=E$, $\Vert q\Vert=F$ and $\Vert r\Vert=\bigcup\limits_{s'\in\mathcal B(s)} {f(s',F)}$. Then, since $\bigcup\limits_{s'\in\mathcal B(s)} {f(s',\Vert p\Vert)}\subseteq \Vert q\Vert$, $q\in K_s\diamond p$ and, since $\bigcup\limits_{s'\in\mathcal B(s)} {f(s',\Vert q\Vert)}\subseteq \Vert p\Vert$, $p\in K_s\diamond q$ and, since $\Vert p\wedge q\Vert\ne\varnothing$, $\top\in K\diamond(p\wedge q)$ (see Remark \ref{REM:topinKdiamond}). Furthermore, $r\in K_s\diamond q$, but, since 
$\bigcup\limits_{s'\in\mathcal B(s)} {f(s',\Vert p\Vert))} \,\nsubseteq\,\Vert r\Vert$, $r\notin K_s\diamond p$. Thus $K_s\diamond p\ne K_s\diamond q$.\\[5pt]
In Case 2, let $p,q,r\in\texttt{At}$ be atomic sentences and construct a model based on this frame where $\Vert p\Vert=E$, $\Vert q\Vert=F$ and $\Vert r\Vert=\bigcup\limits_{s'\in\mathcal B(s)} {f(s',E)}$. Then, $q\in K_s\diamond p$ and $p\in K_s\diamond q$ and $r\in K_s\diamond p$ and (since $\Vert p\wedge q\Vert\ne\varnothing$) $\top\in K\diamond(p\wedge q)$ (see Remark \ref{REM:topinKdiamond}) but, since 
$\bigcup\limits_{s'\in\mathcal B(s)} {f(s',\Vert q\Vert))} \,\nsubseteq\,\Vert r\Vert$, $r\notin K_s\diamond q$. Thus $K_s\diamond p\ne K_s\diamond q$.\medskip
\end{proof}
\begin{proposition}
Frame property
\[(P\diamond7s)\quad\begin{array}{l}
\forall s\in S,\forall E,F\in 2^S\setminus\varnothing\\[4pt]
\bigcup\limits_{s'\in\mathcal B(s)} {f(s',E\cup F)}\,
 \subseteq\, \left(\bigcup\limits_{s'\in\mathcal B(s)} {f(s',E)}\right)\,\,\cup \,\left(\bigcup\limits_{s'\in\mathcal B(s)} {f(s',F)}\right)
 \end{array}\]
characterizes the following KM axiom:
\[(K\diamond7s)\qquad (K\diamond\phi)\cap(K\diamond\psi)\subseteq K\diamond(\phi\vee\psi)\]
\end{proposition}
\begin{proof}
Fix a frame that satisfies property $(P\diamond 7s)$, an arbitrary model based on it, an arbitrary state $s\in S$ and let $\diamond$ be the partial belief change function based on $K_s$  defined by ($\ref{RI}$). Let $\phi,\psi\in\Phi_0$ be in the domain of $\diamond$ (thus $\Vert\phi\Vert\ne\varnothing$ and $\Vert\psi\Vert\ne\varnothing$) and fix an arbitrary $\chi\in (K_s\diamond\phi)\cap (K_s\diamond\psi)$. Then, by  ($\ref{RI}$),   $\bigcup\limits_{s'\in\mathcal B(s)} {f(s',\Vert\phi\Vert)}\subseteq\Vert\chi\Vert$ and $\bigcup\limits_{s'\in\mathcal B(s)} {f(s',\Vert\psi\Vert)}\subseteq\Vert\chi\Vert$ and thus, by Property $(P\diamond 7s)$ (with $E=\Vert\phi\Vert$ and $F=\Vert\psi\Vert$) and the fact that $\Vert\phi\Vert\cup\Vert\psi\Vert=\Vert\phi\vee\psi\Vert$,
$\bigcup\limits_{s'\in\mathcal B(s)} {f(s',\Vert\phi\vee\psi\Vert)}\subseteq\Vert\chi\Vert$, that is, by ($\ref{RI}$), $\chi\in K_s\diamond(\phi\vee\psi)$.\\[3pt]
Conversely, fix a frame that violates property $(P\diamond 7s)$. Then there exist $s\in S$ and $E,F\in 2^S\setminus\varnothing$ such that
\begin{equation}
\label{fornegPKM7s}
\bigcup\limits_{s'\in\mathcal B(s)} {f(s',E\cup F)}\,\nsubseteq \, \bigcup\limits_{s'\in\mathcal B(s)} {f(s',E)}\,\,\cup \,\,\bigcup\limits_{s'\in\mathcal B(s)} {f(s',F)}
\end{equation}
Let $p,q,r\in\texttt{At}$ be atomic formulas and construct a model where $\Vert p\Vert=E$, $\Vert q\Vert=F$ and $\Vert r\Vert=\bigcup\limits_{s'\in\mathcal B(s)} {f(s',E)}\,\,\cup \,\,\bigcup\limits_{s'\in\mathcal B(s)} {f(s',F)}$. Then, by  \eqref{fornegPKM7s} (since $\Vert p\Vert\cup\Vert q\Vert=\Vert p\vee q\Vert$), $\bigcup\limits_{s'\in\mathcal B(s)}{f(s',\Vert p\vee q\Vert)}\nsubseteq\Vert r\Vert$ and thus, by ($\ref{RI}$), $r\notin K_s\diamond(p\vee q)$. On the other hand, since $\bigcup\limits_{s'\in\mathcal B(s)}{f(s',\Vert p\Vert)}\subseteq \Vert r\Vert$ and $\bigcup\limits_{s'\in\mathcal B(s)}{f(s',\Vert q\Vert)}\subseteq \Vert r\Vert$, $r\in K\diamond p$ and $r\in K\diamond q$, yielding a violation of axiom $(K\diamond 7s)$.
\end{proof}
\vspace{8pt}\noindent
\textbf{Proofs of the syntactic characterizations listed in Figure \ref{Fig_2}}.
\par\noindent
The characterizations of  $(A^{*2}_{\diamond1})$,  $(A^{*5b}_{\diamond3b})$ and  $(A^{*7}_{\diamond5})$ are proved in \cite{Bon25b}.\footnote{Since $(K\diamond1)$ coincides with AGM axiom $(K\ast2)$, $(K\diamond3b)$ coincides with AGM axiom $(K\ast5b)$ and $(K\diamond5)$ coincides with AGM axiom $(K\ast7)$.} 
\par\noindent
For $(A\diamond2)$, $(A\diamond6w)$ and $(A\diamond7s)$ the proofs are given in the following three propositions.
\begin{proposition}
The modal formula 
\begin{equation*}
(A\diamond2)\qquad B\phi\,\rightarrow\,\big(B\psi)\leftrightarrow B(\phi>\psi)\big)  
\end{equation*}
is characterized by the following property of frames:
\begin{equation*}
\arraycolsep=10pt\begin{array}{ll}(P\diamond2)&\forall s\in S, \forall E\in 2^S\setminus\varnothing,\text{ if }\mathcal B(s)\subseteq E,\text{ then }\bigcup\limits_{s'\in\mathcal B(s)}{f(s',E)}=\mathcal B(s)
\end{array} 
\end{equation*}
\end{proposition}
\begin{proof}
(A) Fix an arbitrary model based on a frame that satisfies Property $(P\diamond 2)$, an arbitrary state $s\in S$ and an arbitrary formula $\phi\in\Phi_0$ and suppose that $s\models B\phi$, that is, $\mathcal B(s)\subseteq \Vert\phi\Vert$ (thus, by seriality of $\mathcal B$, $\Vert\phi\Vert\ne\varnothing$). We need to show that, for every formula $\psi\in\Phi_0$, $s\models B\psi \leftrightarrow B(\phi>\psi)$, that is, $s\models B\psi$ if and only if $s\models B(\phi>\psi)$. 
By Property $(P\diamond 2)$ (with $E=\Vert\phi\Vert$), $\bigcup\limits_{s'\in\mathcal B(s)}{f(s',\Vert\phi\Vert)}=\mathcal B(s)$ and thus,  $\mathcal B(s)\subseteq \Vert\psi\Vert$ (that is, $s\models B\psi$) if and only if  
$\bigcup\limits_{s'\in\mathcal B(s)}{f(s',\Vert\phi\Vert)}\subseteq\Vert\psi\Vert$ (that is $s\models B(\phi>\psi)$).\smallskip 
\par\noindent
(B) Fix a frame that violates property $(P\diamond2)$. Then there exist a state $s\in S$ and an event $E\subseteq S$ such that (a) $\mathcal B(s)\subseteq E$, (b) $\bigcup\limits_{s'\in\mathcal B(s)} {f(s',E)}\ne \mathcal B(s)$. Two cases are possible.\\[5pt]
CASE 1: $\bigcup\limits_{s'\in\mathcal B(s)}{f(s',E)}\not\subseteq\mathcal B(s)$. Let $p,q\in\texttt{At}$ be atomic formulas and construct a model where $\Vert p\Vert=E$ and $\Vert q\Vert=\mathcal B(s)$. Then $s\models Bp$ and $s\models Bq$ but, since  $\bigcup\limits_{s'\in\mathcal B(s)}{f(s',\Vert p\Vert)}\not\subseteq\Vert q\Vert$, $s\not\models B(p>q)$ so that $s\not\models (Bq\leftrightarrow B(p>q)$. Hence $s\not\models Bp\rightarrow\left(Bq\leftrightarrow B(p>q)\right)$.\\[5pt]
CASE 2: $\mathcal B(s)\not\subseteq\bigcup\limits_{s'\in\mathcal B(s)} {f(s',E)}$. Let $p,q\in\texttt{At}$ be atomic formulas and construct a model where $\Vert p\Vert=E$ and $\Vert q\Vert=\bigcup\limits_{s'\in\mathcal B(s)} {f(s',E)}=\bigcup\limits_{s'\in\mathcal B(s)} {f(s',\Vert p\Vert)}$. Then, $s\models Bp$ and $s\models B(p>q)$ but, since $\mathcal B(s)\not\subseteq \Vert q\Vert$, $s\not\models q$ so that $s\not\models Bp\rightarrow\left(Bq\leftrightarrow B(p>q)\right)$.
\end{proof}
\begin{proposition}
The modal formula 
\begin{equation*}
\arraycolsep=10pt
\begin{array}{ll}
(A\diamond6w)& \neg\square\neg(\phi \wedge\psi)\wedge B(\phi>\psi)\wedge B(\psi>\phi)\\[10pt]
{}&\rightarrow\left(B(\phi>\chi)\leftrightarrow B(\psi>\chi)\right) 
\end{array} 
\end{equation*}
is characterized by the following property of frames:
\begin{equation*}
\arraycolsep=10pt\
\begin{array}{ll}
 (P\diamond 6w)&\forall s\in S,\forall E,F\in 2^S\text{ with }E\cap F\ne\varnothing,\,\text{ if } \bigcup\limits_{s'\in\mathcal B(s)} {f(s',E)} \subseteq F \\[12pt]
 {}&\text{and }\bigcup\limits_{s'\in\mathcal B(s)} {f(s',F)}\subseteq E\, \text{ then } \bigcup\limits_{s'\in\mathcal B(s)} {f(s',E)} =\bigcup\limits_{s'\in\mathcal B(s)} {f(s',F)}
 \end{array} 
\end{equation*}
\end{proposition}
\begin{proof}
(A) Fix an arbitrary model based on a frame that satisfies Property $(P\diamond 6w)$, an arbitrary state $s\in S$ and arbitrary formulas $\phi,\psi,\chi\in\Phi_0$ and suppose that $s\models \neg\square\neg(\phi \wedge\psi)\wedge B(\phi>\psi)\wedge B(\psi>\phi)$, that is, $\Vert\phi\wedge\psi\Vert\ne\varnothing$ (and thus $\Vert\phi\Vert\ne\varnothing$ and $\Vert\psi\Vert\ne\varnothing$), $\bigcup\limits_{s'\in\mathcal B(s)}{f(s',\Vert\phi\Vert)}\subseteq\Vert\psi\Vert$ and $\bigcup\limits_{s'\in\mathcal B(s)}{f(s',\Vert\psi\Vert)}\subseteq\Vert\phi\Vert$. We need to show that $s\models B(\phi>\chi)\leftrightarrow B(\psi>\chi)$. By Property $(P\diamond 6w)$ (with $E=\Vert\phi\Vert$ and $F=\Vert\psi\Vert$ so that $E\cap F=\Vert\phi\Vert\cap\Vert\psi\Vert=\Vert\phi\wedge\psi\Vert\ne\varnothing$), 
\begin{equation}
 \label{U=U}
 \bigcup\limits_{s'\in\mathcal B(s)} {f(s',\Vert\phi\Vert)} =\bigcup\limits_{s'\in\mathcal B(s)} {f(s',\Vert\psi\Vert)}
  \end{equation}
  Suppose that $s\models B(\phi>\chi)$. Then $\bigcup\limits_{s'\in\mathcal B(s)} {f(s',\Vert\phi\Vert)}\subseteq \Vert\chi\Vert$ and thus, by \eqref{U=U}, $\bigcup\limits_{s'\in\mathcal B(s)} {f(s',\Vert\psi\Vert)}\subseteq \Vert\chi\Vert$, that is, $s\models B(\psi>\chi)$. Conversely, suppose that $s\models B(\psi>\chi)$. Then $\bigcup\limits_{s'\in\mathcal B(s)} {f(s',\Vert\psi\Vert)}\subseteq \Vert\chi\Vert$ and thus, by \eqref{U=U}, $\bigcup\limits_{s'\in\mathcal B(s)} {f(s',\Vert\phi\Vert)}\subseteq \Vert\chi\Vert$, that is, $s\models B(\phi>\chi)$. Thus, $s\models B(\phi>\chi)\leftrightarrow B(\psi>\chi)$.\\[3pt]
  (B) Conversely, fix a frame that violates Property $(P\diamond 6w)$. Then there exist $s\in S$, $E,F\in 2^S\text{ with }E\cap F\ne\varnothing$ such that $\bigcup\limits_{s'\in\mathcal B(s)} {f(s',E)}\subseteq F$ and $\bigcup\limits_{s'\in\mathcal B(s)} {f(s',F)}\subseteq E$ but  $\bigcup\limits_{s'\in\mathcal B(s)} {f(s',E)}\ne \bigcup\limits_{s'\in\mathcal B(s)} {f(s',F)}$.  One of the following two cases must hold.\\[4pt]
  Case 1: $\bigcup\limits_{s'\in\mathcal B(s)} {f(s',E)}\not\subseteq \bigcup\limits_{s'\in\mathcal B(s)} {f(s',F)}$, or\\[4pt]
  Case 2: $\bigcup\limits_{s'\in\mathcal B(s)} {f(s',F)}\not\subseteq \bigcup\limits_{s'\in\mathcal B(s)} {f(s',E)}$.\\[4pt]
  In Case 1 let $p$, $q$ and $r$ be atomic formulas and construct a model where $\Vert p\Vert=E$, $\Vert q \Vert=F$ and $\Vert r\Vert= \bigcup\limits_{s'\in\mathcal B(s)} {f(s',F)}$. Since $\varnothing\ne\Vert p\Vert\cap\Vert q\Vert=\Vert p\wedge q\Vert$, $s\models \neg\square\neg (p\wedge q)$; furthermore, $s\models B(p>q)\wedge B(q>p)\wedge B(q>r)$ but $s\not\models B(p>r)$.\\[4pt]
  In Case 2 construct a model where $\Vert p\Vert=E$, $\Vert q \Vert=F$ and $\Vert r\Vert= \bigcup\limits_{s'\in\mathcal B(s)} {f(s',E)}$. Again, since $\varnothing\ne\Vert p\Vert\cap\Vert q\Vert=\Vert p\wedge q\Vert$, $s\models \neg\square\neg (p\wedge q)$; furthermore, $s\models B(p>q)\wedge B(q>p)\wedge B(p>r)$ but $s\not\models B(q>r)$.
\end{proof}
\par\noindent
\begin{proposition}
The modal formula
\begin{equation*}
\arraycolsep=10pt
\begin{array}{ll}
(A\diamond7s)& \neg\square\neg\phi \wedge\neg\square\neg\psi\wedge B(\phi>\chi)\wedge B(\psi>\chi)\\[10pt]
{}&\rightarrow B\left((\phi\vee\psi)>\chi\right) 
\end{array} 
\end{equation*}
is characterized by the following property of frames:
\begin{equation*}
\begin{array}{ll}(P\diamond 7s)&
\forall s\in S,\forall E,F\in 2^S\setminus\varnothing\\[8pt]
{}&\bigcup\limits_{s'\in\mathcal B(s)} {f(s',E\cup F)}\, \subseteq\, \left(\bigcup\limits_{s'\in\mathcal B(s)} {f(s',E)}\right)\,\,\cup \,\left(\bigcup\limits_{s'\in\mathcal B(s)} {f(s',F)}\right)
\end{array}
\end{equation*}
\end{proposition}
\begin{proof}
Fix a frame that satisfies property $(P\diamond 7s)$, an arbitrary model based on it, an arbitrary state $s\in S$ and arbitrary $\phi,\psi,\chi\in\Phi_0$ and assume that $s\models \neg\square\neg\phi \wedge\neg\square\neg\psi\wedge B(\phi>\chi)\wedge B(\psi>\chi)$. Then $\Vert\phi\Vert\ne\varnothing$, $\Vert\psi\Vert\ne\varnothing$, $\bigcup\limits_{s'\in\mathcal B(s)}{f(s',\Vert \phi\Vert)}\subseteq \Vert \chi\Vert$ and $\bigcup\limits_{s'\in\mathcal B(s)}{f(s',\Vert \psi\Vert)}\subseteq \Vert \chi\Vert$. Thus, by property $(P\diamond 7s)$ (with $E=\Vert\phi\Vert$ and $F=\Vert\psi\Vert$ and using the fact that $\Vert\phi\Vert\cup\Vert\psi\Vert=\Vert\phi\vee\psi\Vert$), $\bigcup\limits_{s'\in\mathcal B(s)}{f(s',\Vert \phi\vee\psi\Vert)}\subseteq \Vert \chi\Vert$, so that $s\models B\left((\phi\vee\psi)>\chi\right)$.\\[3pt]
Conversely, fix a frame that violates property $(P\diamond 7s)$. Then there exist $s,\in S$ and $E,F\in 2^S\setminus\varnothing$ such that
\begin{equation}
\label{fornegA7s}
\bigcup\limits_{s'\in\mathcal B(s)} {f(s',E\cup F)}\,\nsubseteq \, \bigcup\limits_{s'\in\mathcal B(s)} {f(s',E)}\,\,\cup \,\,\bigcup\limits_{s'\in\mathcal B(s)} {f(s',F)}
\end{equation}
Let $p,q,r\in\texttt{At}$ be atomic formulas and construct a model where $\Vert p\Vert=E$, $\Vert q\Vert=F$ and $\Vert r\Vert=\bigcup\limits_{s'\in\mathcal B(s)} {f(s',E)}\,\,\cup \,\,\bigcup\limits_{s'\in\mathcal B(s)} {f(s',F)}$. Then, since $\Vert p\Vert\ne\varnothing$ and  $\Vert q\Vert\ne\varnothing$, $s\models \neg\square\neg p\wedge \neg\square\neg q$; furthermore, $s\models B(p>r)\wedge B(q>r)$. On the other hand, since (noting that $\Vert p\Vert\cup\Vert q\Vert=\Vert p\vee q\Vert$) $\bigcup\limits_{s'\in\mathcal B(s)} {f(s',\Vert p\vee q\Vert)}\,\nsubseteq \, \Vert r\Vert$, $s\not\models B\left((p\vee q)> r\right)$, yielding a violation of axiom $(A\diamond7s)$.
\end{proof}
\par\noindent
\textbf{Proof of Proposition \ref{PROP:sublogic}.} \emph{Every axiom of $\mathcal L_{KM}$ is a theorem of $\mathcal L_{AGM}$.}

Since some of the AGM axioms coincide with KM axioms, we only need to prove that $(A\diamond2)$, $(A\diamond6w)$ and $(A\diamond7s)$ are theorems of  $\mathcal L_{AGM}$. This is done in the following three lemmas. In the proofs 'PL' stands for 'Propositional Logic'.
\begin{lemma}
\label{LEM:forKdiamond2}
The following modal KM axiom: 
\[(A\diamond2)\qquad B\phi\,\rightarrow\,\left(B\psi\leftrightarrow B(\phi>\psi)\right)\]
is a theorem of $\mathcal L_{AGM}$.
\end{lemma}
\begin{proof}
We first prove that $B\phi\,\rightarrow\,\left(B\psi\rightarrow B(\phi>\psi)\right)$ is a theorem of $\mathcal L_{AGM}$.\\[5pt]
$\arraycolsep=10pt \begin{array}{lll}
1.&B\phi\rightarrow\neg B\neg \phi&\text{axiom }D_B\\
2.&\psi\rightarrow(\phi\rightarrow\psi)&\text{tautology}\\
3.& B\psi\rightarrow B\left( \phi\rightarrow\psi\right)& \text{rule }(RM_B)\\
4.&(B\phi\wedge B\psi)\rightarrow\left( \neg B\neg\phi\wedge B(\phi\rightarrow\psi)\right)&\text{1, 3, PL}\\
5.&\left(\neg B\neg\phi\wedge B(\phi\rightarrow\psi)\right)\rightarrow B(\phi>\psi)&\text{AGM axiom }(A*4)\\
6.& (B\phi\wedge B\psi)\rightarrow B(\phi>\psi)&\text{4, 5, PL.}\\
7.&B\phi\,\rightarrow\,\left(B\psi\rightarrow B(\phi>\psi)\right)&\text{6, PL}
\end{array}$\\[5pt]
Next we prove that $B\phi\,\rightarrow\,\left(B(\phi>\psi)\rightarrow B\psi \right)$ is a theorem of $\mathcal L_{AGM}$.\\[5pt]
$\arraycolsep=10pt\begin{array}{lll}
8.&\square\neg\phi\rightarrow B\neg\phi &\text{axiom }(NB)\\
9.&\neg B\neg\phi \rightarrow\neg \square\neg\phi&\text{8, PL}\\
10.& B\phi\rightarrow\neg\square\neg\phi&\text{1, 9, PL}\\
11.&\left(B\phi\wedge B(\phi>\psi)\right)\,\rightarrow\,\left(\neg\square\neg\phi\wedge B(\phi>\psi)\right)&\text{10, PL}\\
12.&\left(\neg\square\neg\phi\wedge B(\phi>\psi)\right)\rightarrow B(\phi\rightarrow\psi)&\text{AGM axiom }(A*3)\\
13.&\left(B\phi\wedge B(\phi>\psi)\right)\,\rightarrow\,B(\phi\rightarrow\psi)&\text{11, 12, PL}
\end{array}$
\par\noindent
$\arraycolsep=10pt\begin{array}{lll}
14.&\left(B\phi\wedge B(\phi>\psi)\right)\,\rightarrow\,\left(B\phi\wedge B(\phi\rightarrow\psi)\right)&\text{13, PL}\\
15.&\left(B\phi\wedge B(\phi\rightarrow\psi)\right)\,\rightarrow\,B\left(\phi\wedge(\phi\rightarrow\psi)\right)&\text{axiom }(C_B)\\
16.&\phi\wedge(\phi\rightarrow\psi)\,\rightarrow\,\psi&\text{tautology}\\
17.&B(\phi\wedge(\phi\rightarrow\psi))\,\rightarrow\,B\psi&\text{16, rule }(RM_B)\\
18.&\left(B\phi\wedge B(\phi>\psi)\right)\,\rightarrow\,B\psi&\text{14, 15, 17, PL}\\
19.&B\phi\,\rightarrow\,\big(B(\phi>\psi)\rightarrow B\psi \big)&\text{18, PL}
\end{array}$\\[5pt]
Axiom $(A\diamond2)$ follows from lines 7 and 19.
\end{proof}
\par\noindent
\begin{lemma}
The following modal KM axiom: 
\[(A\diamond6w)\qquad \neg\square\neg(\phi\wedge\psi)\wedge
B(\phi>\psi)\wedge B(\psi>\phi)
\,\rightarrow\,\left(B(\phi>\chi)\leftrightarrow B(\psi>\chi)\right)\]
is a theorem of $\mathcal L_{AGM}$.
\end{lemma}
\begin{proof}
First we show that the following formula is a theorem of $\mathcal L_{AGM}$:
\[\big(\neg\square\neg(\phi\wedge\psi)\wedge B(\phi>\psi)\big)\,\rightarrow\,\big(B(\phi>\chi)\rightarrow B\left((\phi\wedge\psi)>\chi)\right)\]
\renewcommand{\arraystretch}{1.5}
    \begin{longtable}{lMM}
  1.&\neg\square\neg(\phi\wedge\psi)\,\rightarrow \, \neg\square\neg\phi&(C_{\neg\square\neg})\\[-3pt] {}&{}&\text{(see Remark \ref{REM:derivedAxRul})}\\
2.&\big(\neg\square\neg(\phi\wedge\psi)\wedge B(\phi>\psi)\big)
\,\rightarrow\,
\left(\neg\square\neg\phi\wedge B(\phi>\psi)\right)&\text{1, PL}\\
3. &\left(\neg\square\neg\phi\wedge B(\phi>\psi)\right)\rightarrow
\neg B(\phi>\neg\psi)&\text{axiom }(A^{*5b}_{\diamond3b})\\
4.&\big(\neg\square\neg(\phi\wedge\psi)\wedge B(\phi>\psi)\big)
\,\rightarrow\,
\neg B(\phi>\neg\psi)&\text{2, 3, PL}\\
5. & \chi\rightarrow(\psi\rightarrow\chi)&\text{tautology}\\
6.&B(\phi>\chi)\rightarrow B\left(\phi>(\psi\rightarrow\chi)\right)&\text{5, rule }(RM_{B>})\\[-3pt]
{}&{}&\text{(see Remark \ref{REM:derivedAxRul})}\\
7.& \neg\square\neg(\phi\wedge\psi)\wedge B(\phi>\psi)\wedge B(\phi>\chi)\\
{}&\rightarrow\,
\neg B(\phi>\neg\psi)\wedge B\left(\phi>(\psi\rightarrow\chi)\right)&
\text{4, 6, PL}\\
8.& \neg B(\phi>\neg\psi)\wedge B\left(\phi>(\psi\rightarrow\chi)\right)\\
{}&\rightarrow\,
B\big((\phi\wedge\psi)>(\psi\wedge\chi)\big)&
\text{axiom }(A*8)\\
9.&\neg\square\neg(\phi\wedge\psi)\wedge B(\phi>\psi)\wedge B(\phi>\chi)\\
{}&\rightarrow\,
B\big((\phi\wedge\psi)>(\psi\wedge\chi)\big)&\text{7, 8, PL}\\
10.&(\psi\wedge\chi)\rightarrow\chi&\text{tautology}\\
11.& B\big((\phi\wedge\psi)>(\psi\wedge\chi)\big)
\,\rightarrow\,
B\big((\phi\wedge\psi)>\chi\big)&\text{10, rule }(RM_{B>})\\[-3pt]
{}&{}&\text{(see Remark \ref{REM:derivedAxRul})}\\
12.&\left(\neg\square\neg(\phi\wedge\psi)\wedge B(\phi>\psi)\wedge B(\phi>\chi)\right)
\,\rightarrow\,
B\big((\phi\wedge\psi)>\chi\big)&\text{9, 11, PL}\\
13.&\left(\neg\square\neg(\phi\wedge\psi)\wedge B(\phi>\psi)\right)
\,\rightarrow\,
\big(B(\phi>\chi)\rightarrow B\big((\phi\wedge\psi)>\chi\big)\big)&\text{12, PL}\medskip
\end{longtable}
\par\noindent
Next, repeating steps 1-13 with $\phi$ replaced by $\psi$ and with $\psi$ replaced by $\phi$ we get\medskip
\par\noindent
$\begin{array}{ll}
14.&\left(\neg\square\neg(\phi\wedge\psi)\wedge B(\psi>\phi)\right)
\,\rightarrow\,
\big(B(\psi>\chi)\rightarrow B\big((\phi\wedge\psi)>\chi\big)\big)\medskip
\end{array}$

\par\noindent
The next step is to show that the following is a theorem of $\mathcal L_{AGM}$:
\begin{equation*}
  \big(\neg\square\neg(\phi\wedge\psi)\wedge B(\phi>\psi)\big)\,\rightarrow\,\big(B\left((\phi\wedge\psi)>\chi\right) \rightarrow B(\phi>\chi)\big) 
  \end{equation*} 

\renewcommand{\arraystretch}{1.5}
    \begin{longtable}{lMM}
15.&\neg\square\neg(\phi\wedge\psi)\wedge B\left((\phi\wedge\psi)>\chi\right)&{}\\
{}&\rightarrow\,
B\big(\phi>(\psi\rightarrow\chi)\big)&\text{axiom }(A*7)\\
16.&\neg\square\neg(\phi\wedge\psi)\wedge B\left((\phi\wedge\psi)>\chi\right)\wedge B(\phi>\psi)&{}\\
{}&\rightarrow\,
\big( B(\phi>\psi)\wedge B\big(\phi>(\psi\rightarrow\chi)\big) \big)&\text{15, PL}\\
17.&\big(B(\phi>\psi)\wedge B(\phi>(\psi\rightarrow\chi)\big)
\,\rightarrow\, B(\phi>\chi)&\text{axiom }(A^{*1}_{\diamond 0})\\
18.&\big(\neg\square\neg(\phi\wedge\psi)\wedge B(\phi>\psi)\wedge B\left((\phi\wedge\psi)>\chi\right)\big)
\,\rightarrow\, B(\phi>\chi)&\text{16, 17, PL}\\
19.&\big(\neg\square\neg(\phi\wedge\psi)\wedge B(\phi>\psi)\big)
\,\rightarrow\, 
\big(B\left((\phi\wedge\psi)>\chi\right)\rightarrow B(\phi>\chi)\big)&\text{18, PL}
\end{longtable}
\noindent
Next, repeating steps 15-19 with $\phi$ replaced by $\psi$ and with $\psi$ replaced by $\phi$ we get
\begin{equation*}
\begin{array}{lll}
20.&\big(\neg\square\neg(\phi\wedge\psi)\wedge B(\psi>\phi)\big)
\,\rightarrow\, 
\big(B\left((\phi\wedge\psi)>\chi\right)\rightarrow B(\psi>\chi)\big)&{}
\end{array}
\end{equation*}
From 13 and 19 we get\medskip
\par\noindent
$\begin{array}{ll}
21.&\big(\neg\square\neg(\phi\wedge\psi)\wedge B(\phi>\psi)\big)
\,\rightarrow\, 
\big(B\left((\phi\wedge\psi)>\chi\right)\leftrightarrow B(\phi>\chi)\big)
\end{array}$
\medskip
\par\noindent
and from 14 and 20 we get\medskip
\par\noindent
$\begin{array}{ll}
22.&\big(\neg\square\neg(\phi\wedge\psi)\wedge B(\psi>\phi)\big)
\,\rightarrow\, 
\big(B\left((\phi\wedge\psi)>\chi\right)\leftrightarrow B(\psi>\chi)\big)
\end{array}$
\medskip
\par\noindent
Thus,\medskip
\par\noindent
\begin{equation*}
\begin{array}{lll}
23.&\big(\neg\square\neg(\phi\wedge\psi)\wedge B(\phi>\psi)\wedge B(\psi>\phi)\big)&{}\\
{}&\rightarrow\,\Big(\big(B\left((\phi\wedge\psi)>\chi\right)\leftrightarrow B(\phi>\chi)\big)&{}\\
{}&\qquad\wedge\big(B\left((\phi\wedge\psi)>\chi\right)\leftrightarrow B(\psi>\chi)\big)\Big)&\text{21, 22, PL}\\[6pt]
24.&\big(B\left((\phi\wedge\psi)>\chi\right)\leftrightarrow B(\phi>\chi)\big) \wedge\big(B\left((\phi\wedge\psi)>\chi\right)\leftrightarrow B(\psi>\chi)\big)&{}\\[6pt]
{}&\rightarrow\,\left(B(\phi>\chi)\leftrightarrow B(\psi>\chi)\right)&\text{tautology}\\[8pt]
25.&\neg\square\neg(\phi\wedge\psi)\wedge B(\phi>\psi)\wedge B(\psi>\phi)&{}\\[6pt]
{}&\rightarrow\, \big(B(\phi>\chi)\leftrightarrow B(\psi>\chi)\big)
&\text{23, 24, PL;}\\
{}&{}&\text{this is }(A\diamond 6) 
\end{array}
\end{equation*}
\end{proof}
\begin{lemma}
The following modal KM axiom: 
\[(A\diamond7s)\qquad \neg\square\neg\phi \wedge \neg\square\neg\psi\wedge B(\phi>\chi)\wedge B(\psi>\chi)\\[-5pt]
\rightarrow B\left((\phi\vee\psi)>\chi\right)\]
is a theorem of $\mathcal L_{AGM}$.
\end{lemma}
\begin{proof}
First we prove that the following formula is a theorem of $\mathcal L_{AGM}$:
\[\neg\square\neg\phi\wedge B(\phi>\chi)\,\rightarrow\,B\big((\phi\vee\psi)>(\phi\rightarrow\chi)\big)\]
\renewcommand{\arraystretch}{1.5}
\begin{longtable}{lMM}
1.&\phi\leftrightarrow\big((\phi\vee\psi)\wedge\phi\big)&\text{tautology}\\
2.&B(\phi>\chi)\leftrightarrow B\big(\left((\phi\vee\psi)\wedge\phi\right)>\chi\big)&\text{1, rule }(R^{*6}_{\diamond4})\\
3.&B(\phi>\chi)\rightarrow B\big(\left((\phi\vee\psi)\wedge\phi\right)>\chi\big)&\text{2, PL}\\
4.&\phi\rightarrow\big((\phi\vee\psi)\wedge\phi\big)&\text{tautology}\\
5.&\neg\square\neg\phi \rightarrow \neg\square\neg\left((\phi\vee\psi)\wedge\phi\right)&\text{1, rule }(RM_{\neg\square\neg})\\[-5pt]
{}&{}&(\text{see Remark \ref{REM:derivedAxRul}})\\
6.&\neg\square\neg\phi\wedge B(\phi>\chi)\,{}\\
{}&\rightarrow\,\neg\square\neg\left((\phi\vee\psi)\wedge\phi\right)  \,\wedge\, B\big(\left((\phi\vee\psi)\wedge\phi\right)>\chi\big)&{3, 5, PL}\\
7.&\neg\square\neg\left((\phi\vee\psi)\wedge\phi\right)  \,\wedge\, B\big(\left((\phi\vee\psi)\wedge\phi\right)>\chi\big)&{}\\
{}&\rightarrow\,B\big((\phi\vee\psi)>(\phi\rightarrow\chi)\big)&\text{axiom }(A^{*7}_{\diamond5})\\
8.&\neg\square\neg\phi\wedge B(\phi>\chi)\,\rightarrow\,B\big((\phi\vee\psi)>(\phi\rightarrow\chi)\big)&\text{6, 7, PL}\medskip
\end{longtable}
\par\noindent
Repeating steps 1-8 with $\phi$ replaced by $\psi$ and $\psi$ by $\phi$ we get\medskip
\par\noindent
\quad$9.\quad\neg\square\neg\psi\wedge B(\psi>\chi)\,\rightarrow\,B\big((\phi\vee\psi)>(\psi\rightarrow\chi)\big)$
Thus,
\renewcommand{\arraystretch}{1.5}
\begin{longtable}{lMM}
10.&\neg\square\neg\phi\wedge B(\phi>\chi)\wedge \neg\square\neg\psi\wedge B(\psi>\chi)&{}\\
{}&\rightarrow\,B\big((\phi\vee\psi)>(\phi\rightarrow\chi)\big)\wedge B\big((\phi\vee\psi)>(\psi\rightarrow\chi)\big)&\text{8, 9, PL}\\
11.&B\big((\phi\vee\psi)>(\phi\rightarrow\chi)\big)\wedge B\big((\phi\vee\psi)>(\psi\rightarrow\chi)\big)\\
{}&\rightarrow B\big((\phi\vee\psi)>(\phi\rightarrow\chi\wedge\psi\rightarrow\chi)\big)&\text{axiom }(C_B)\\
12.&(\phi\rightarrow\chi)\wedge(\psi\rightarrow\chi) \,\rightarrow \left((\phi\vee\psi)\rightarrow\chi\right)&\text{tautology}\\
13.&B\big((\phi\vee\psi)>(\phi\rightarrow\chi\wedge\psi\rightarrow\chi)\big)&{}\\
{}&\rightarrow B\big((\phi\vee\psi)>\left((\phi\vee\psi)\rightarrow\chi\right) \big)&\text{12, rule }(RM_{B>})\\[-2pt]
{}&{}&\text{(see Remark \ref{REM:derivedAxRul})}\\
14.&\neg\square\neg\phi\wedge B(\phi>\chi)\wedge \neg\square\neg\psi\wedge B(\psi>\chi)&{}\\
{}&\rightarrow\,B\big((\phi\vee\psi)>\left((\phi\vee\psi)\rightarrow\chi\right) \big)&\text{10, 11, 13, PL}\\
15.&B\big((\phi\vee\psi)>(\phi\vee\psi)\big)&\text{axiom }(A^{*2}_{\diamond1})\\
16.&B\big((\phi\vee\psi)>\left((\phi\vee\psi)\rightarrow\chi\right)\big)&{}\\
{}& \rightarrow B\big((\phi\vee\psi)>(\phi\vee\psi)\big)\wedge B\big((\phi\vee\psi)>\left((\phi\vee\psi)\rightarrow\chi\right)\big)
&\text{15, PL}\\
17.&B\big((\phi\vee\psi)>(\phi\vee\psi)\big)\wedge B\big((\phi\vee\psi)>((\phi\vee\psi)\rightarrow\chi)\big)&{}\\
{}&\rightarrow\, B\big((\phi\vee\psi)>\chi\big)&\text{axiom }(A^{*1}_{\diamond0})\\[-3pt]
{}&{}&(\text{with }\phi \text{ and } \psi\\[-4pt]
{}&{}&\text{replaced by }\phi\vee\psi)\\
18.&B\big((\phi\vee\psi)>\left((\phi\vee\psi)\rightarrow\chi\right)\big)\,\rightarrow B\big((\phi\vee\psi)>\chi\big)
&\text{16, 17, PL}\\
19.&\neg\square\neg\phi\wedge B(\phi>\chi)\wedge \neg\square\neg\phi\wedge B(\psi>\chi)&{}\\
{}&\rightarrow\, B\big((\phi\vee\psi)>\chi\big)&\text{14, 18, PL}\\[-4pt]
{}&{}&\text{this is }(A\diamond7s)
\end{longtable}
\end{proof}
\noindent
\textbf{Lemma \ref{LEM:K9s}}. \emph{$(K\diamond9s)$ follows from $(K\diamond0)$, $(K\diamond8)$ and $(K\diamond9)$.}
\begin{proof}
Fix arbitrary $\phi,\psi\in\Phi_0$. By $(K\diamond8)$ (recall that $\llbracket K\rrbracket=\{w\in W: K\subseteq w\}$)
\begin{equation}
\label{forK9sa}
K\diamond\phi\,=\,\bigcap_{w\in\llbracket K\rrbracket}{(w\diamond\phi)}
\end{equation} 
Let $A=\left\{w\in\llbracket K\rrbracket:\neg\psi\notin w\diamond\phi\right\}$ and $B=\left\{w\in\llbracket K\rrbracket:\neg\psi\in w\diamond\phi\right\}$ (clearly $\llbracket K\rrbracket=A\cup B$). Assume that $\neg\phi\notin K\diamond\phi$. Then $A\ne\varnothing$. First note that
\begin{equation}
\label{forK9sb}
\forall w\in B, \,(w\diamond\phi)+\psi=\Phi_0;\text{ thus, }\bigcap_{w\in B}\big((w\diamond\phi)+\psi\big)=\Phi_0
\end{equation}
It follows from \eqref{forK9sb} that
\begin{equation}
\label{forK9sc}
\bigcap_{w\in\llbracket K\rrbracket}\left((w\diamond\phi)+\psi\right)=\bigcap_{w\in A}\left((w\diamond\phi)+\psi\right)\,\cap\,\Phi_0=\bigcap_{w\in A}\left((w\diamond\phi)+\psi\right)
\end{equation}
By $(K\diamond9)$ (since each $w$ is complete), 
\begin{equation}
\label{forK9sd}
\forall w\in A, \quad(w\diamond\phi)+\psi\,\subseteq\,w\diamond(\phi\wedge\psi)
\end{equation}
Next, note that, by $(K\diamond0)$ (which ensures that $K\diamond\phi=Cn\left(K\diamond\phi\right)$ and $w\diamond\phi=Cn\left(w\diamond\phi\right)$) and \eqref{forK9sa},\footnote{
Since $K\diamond\phi$ is deductively closed, for every $\chi\in\Phi_0$, $\chi\in K\diamond\phi+\psi$ if and only if $(\psi\rightarrow\chi)\in K\diamond\phi$. Since, $\forall w\in W$, $w\diamond\phi$ is deductively closed, $\bigcap_{w\in\llbracket K\rrbracket}{(w\diamond\phi)}$ is deductively closed and thus $\chi\in \left(\bigcap_{w\in\llbracket K\rrbracket}{(w\diamond\phi)}\right)+\psi$ if and only if $(\psi\rightarrow\chi)\in \bigcap_{w\in\llbracket K\rrbracket}{(w\diamond\phi})$ if and only if $\chi\in \bigcap_{w\in\llbracket K\rrbracket}{\left(w\diamond\phi+\psi\right)}$.
}
\begin{equation}
\label{forK9se}
(K\diamond\phi)+\psi\,=\,\left(\bigcap_{w\in\llbracket K\rrbracket}{(w\diamond\phi)}\right)+\psi\,=\,\bigcap_{w\in\llbracket K\rrbracket}{(w\diamond\phi+\psi)}
\end{equation}
It follows from \eqref{forK9sc}, \eqref{forK9sd} and \eqref{forK9se} that
 \begin{equation}
\label{forK9sf}
(K\diamond\phi)+\psi\,\subseteq\,\bigcap_{w\in\llbracket K\rrbracket}{(w\diamond(\phi\wedge\psi))}
\end{equation}
and, by $(K\diamond 8)$, $\bigcap_{w\in\llbracket K\rrbracket}{(w\diamond(\phi\wedge\psi))}=K\diamond(\phi\wedge\psi)$. 
\end{proof}
\noindent
\textbf{Lemma \ref{LEM:K*3inKM}}. $(A\ast3)\quad\big(\neg\square\neg\phi\wedge B(\phi>\psi\big)\rightarrow B(\phi\rightarrow\psi)$ \emph{is provable in logic} $\mathcal L_{KM}$.
\begin{proof}
Recall that $(A\diamond2)$ is the following axiom, for $\chi,\xi\in\Phi_0$,  $B\chi\rightarrow\big(B\xi\leftrightarrow B(\chi>\xi)\big)$; in line 1 we take the instance of $(A\diamond2)$ with $\chi=\phi\vee\neg\phi$ and $\xi=\phi\rightarrow\psi$. Recall also that  $(A\diamond5)$ is the following axiom, for $\chi,\xi,\theta\in\Phi_0$, $\big(\neg\square\neg(\chi\wedge\xi)\wedge B\left((\chi\wedge\xi)>\theta\right)\big)\,\rightarrow\, B(\chi>(\xi\rightarrow\theta))$; in line 5 we take the instance of $(A\diamond5)$ with $\chi=\phi\vee\neg\phi$ and $\xi=\phi$ and $\theta=\psi$.
\renewcommand{\arraystretch}{1.5}
\begin{longtable}{lMM}
1.&B(\phi\vee\neg\phi)\,\rightarrow\,\big(B(\phi\rightarrow\psi)\leftrightarrow B\left((\phi\vee\neg\phi)>(\phi\rightarrow\psi)\right)\big)&\text{axiom }(A\diamond 2)\\
2.&\phi\vee\neg\phi&\text{tautology}\\
3.&B(\phi\vee\neg\phi)&\text{2, rule }(N_B)\\
4.&B(\phi\rightarrow\psi)\leftrightarrow B\left((\phi\vee\neg\phi)>(\phi\rightarrow\psi)\right)&\text{1, 3, rule }(MP)\\
5.& \neg\square\neg((\phi\vee\neg\phi)\wedge\phi)\wedge  B\left((\phi\vee\neg\phi)\wedge\phi)>\psi\right)&{}\\
{}&\rightarrow B((\phi\wedge\neg\phi)>(\phi\rightarrow\psi))&\text{axiom }(A\diamond5)\\
6.& \phi\rightarrow (\phi\vee\neg\phi)\wedge\phi&\text{tautology}\\
7.&\neg\square\neg\phi\,\rightarrow\,\neg\square\neg((\phi\vee\neg\phi)\wedge\phi)&\text{6, rule }(RM_{\neg\square\neg})\\
{}&{}&\text{see Remark \ref{REM:derivedAxRul}}\\
8.&B(\phi>\psi)\,\rightarrow\,B\left((\phi\vee\neg\phi)\wedge\phi)>\psi\right)&\text{6, rule }(RM_{B>})\\
{}&{}&\text{see Remark \ref{REM:derivedAxRul}}\\
9.& \neg\square\neg\phi\wedge B(\phi>\psi)&{}\\
{}&\rightarrow\,\neg\square\neg((\phi\vee\neg\phi)\wedge\phi)\wedge  B\left((\phi\vee\neg\phi)\wedge\phi)>\psi\right)&\text{7, 8 PL}\\
10.& \neg\square\neg\phi\wedge B(\phi>\psi)\,\rightarrow\,B((\phi\wedge\neg\phi)>(\phi\rightarrow\psi))&\text{9, 5, PL}\\
11.& \neg\square\neg\phi\wedge B(\phi>\psi)\,\rightarrow\,B(\phi\rightarrow\psi)&\text{10, 4, PL}
\end{longtable}
\end{proof}
We conclude by proving the semantic correspondence for AGM axiom $(K\ast4)$ and its modal counterpart $(A\ast4)$.
\begin{lemma}
\label{LEM:K*4}
The following property of frames:
\[(P\ast4)\begin{array}{l} \forall s\in S, \forall E,F\in 2^S,\\[4pt]
 \text{ if } \mathcal B(s)\cap E\ne\varnothing\text{ and }\mathcal B(s)\subseteq(S\setminus E)\cup F\text{ then } \bigcup\limits_{s'\in\mathcal B(s)}{f(s',E)}\,\subseteq\, F\end{array}\]
 characterizes the AGM axiom
 \[(K\ast4)\quad\text{If } \neg\phi\notin K\text{ then }K+\phi\subseteq K\ast\phi\]
 and its modal counterpart
 \[(A\ast4)\quad \neg B\neg\phi \wedge B(\phi\rightarrow\psi)\,\rightarrow\, B(\phi>\psi)\]
\end{lemma}
\begin{proof}
(A) First we prove that $(P\ast4)$ characterizes $(K\ast4)$. Fix an arbitrary model based on a frame that satisfies Property $(P\ast4)$, an arbitrary state $s$ and an arbitrary formula $\phi$ and suppose that $\neg\phi\notin K_s$, that is, $\mathcal B(s)\cap \Vert\phi\Vert\ne\varnothing$. Let $\psi\in K_s+\phi$; then, since $K_s$ is deductively closed, $\phi\rightarrow\psi\in K_s$, that is, $\mathcal B(s)\subseteq\Vert \phi\rightarrow\psi\Vert=(S\setminus\Vert\phi\Vert)\cup\Vert\psi\Vert$. Then by Property $(P\ast4)$ (with $E=\Vert\phi\Vert$ and $F=\Vert\psi\Vert$) $\bigcup\limits_{s'\in\mathcal B(s)}{f(s',\Vert\phi\Vert)}\,\subseteq\, \Vert\psi\Vert$, that is, $\psi\in K\diamond\phi$.\\[3pt]
Conversely, fix a frame that violates Property $(P\ast4)$. Then there exist $s\in S$ and $E,F\in 2^S$ such that $\mathcal B(s)\cap E\ne\varnothing$ and $\mathcal B(s)\subseteq(S\setminus E)\cup F$ but $\bigcup\limits_{s'\in\mathcal B(s)}{f(s',E)}\,\not\subseteq\, F$. Let $p,q\in \texttt{At}$ be atomic sentences and construct a model based on this frame where $\Vert p\Vert=E$ and $\Vert q\Vert=F$. Then, since  $\mathcal B(s)\cap \Vert p\Vert\ne\varnothing$, $\neg p\notin K_s$ and, since $\mathcal B(s)\subseteq(S\setminus \Vert p\Vert)\cup \Vert q\Vert = \Vert p\rightarrow q\Vert$, $(p\rightarrow q)\in K_s$, that is $q\in K_s+{p}$; however, since  $\bigcup\limits_{s'\in\mathcal B(s)}{f(s',\Vert p\Vert)}\,\not\subseteq\, \Vert q\Vert$, $q\notin K\diamond p$.\\[10pt]
(B) Next we prove that $(P\ast4)$ characterizes $(A\ast4)$.  Fix an arbitrary model based on a frame that satisfies Property $(P\ast4)$, an arbitrary state $s$ and arbitrary formulas $\phi$ and $\psi$ and suppose that $s\models \neg B\neg\phi \wedge B(\phi\rightarrow\psi)$. We need to show that $s\models B(\phi>\psi)$. Since $s\models \neg B\neg\phi$, $\mathcal B(s)\cap\Vert\phi\Vert\ne\varnothing$. Since $s\models B(\phi\rightarrow\psi)$, $\mathcal B(s)\subseteq\Vert\phi\rightarrow\psi\Vert=(S\setminus\Vert\phi\Vert)\cup\Vert\psi\Vert$. Thus, by Property $(P\ast4)$ (with $E=\Vert\phi\Vert)$ and $F=\Vert\psi\Vert)$, $\bigcup\limits_{s'\in\mathcal B(s)}{f(s',\Vert\phi\Vert)}\,\subseteq\, \Vert\psi\Vert$, that is, $s\models B(\phi>\psi)$.\\[3pt]
Conversely,  fix a frame that violates Property $(P\ast4)$. Then there exist $s\in S$ and $E,F\in 2^S$ such that $\mathcal B(s)\cap E\ne\varnothing$ and $\mathcal B(s)\subseteq(S\setminus E)\cup F$ but $\bigcup\limits_{s'\in\mathcal B(s)}{f(s',E)}\,\not\subseteq\, F$. Let $p,q\in \texttt{At}$ be atomic sentences and construct a model based on this frame where $\Vert p\Vert=E$ and $\Vert q\Vert=F$. Then, since  $\mathcal B(s)\cap \Vert p\Vert\ne\varnothing$, $s\models \neg B\neg p$ and, since $\mathcal B(s)\subseteq(S\setminus \Vert p\Vert)\cup \Vert q\Vert = \Vert p\rightarrow q\Vert$, $s\models B(p\rightarrow q)$; however, since  $\bigcup\limits_{s'\in\mathcal B(s)}{f(s',\Vert p\Vert)}\,\not\subseteq\, \Vert q\Vert$, $s\not\models B(p>q).$
\end{proof}

\end{document}